\documentclass{article} %
\usepackage[utf8]{inputenc}

\usepackage{etoolbox}
\newtoggle{nips}
\global\toggletrue{nips} %
\newcommand{\arxiv}[1]{\iftoggle{nips}{}{#1}}
\newcommand{\nips}[1]{\iftoggle{nips}{#1}{}}

\usepackage{txie}

\makeatletter
\AtBeginDocument{%
  \let\cref@genuine@label\label
  \g@addto@macro\@arrayparboxrestore{\let\label\cref@genuine@label}%
}
\makeatother
\nips{\usepackage[main, final]{neurips_2026}}

\usepackage[suppress]{color-edits}

\addauthor[Tengyang]{tx}{teal}
\addauthor[Haoqun]{hq}{blue}
\addauthor[Yurun]{yy}{purple}

\arxiv{
\usepackage{geometry}
\geometry{letterpaper, margin=1in}
}

\makeatletter
\arxiv{
\usepackage{parskip}
\def\thm@space@setup{%
  \thm@preskip=\parskip
  \thm@postskip=\parskip
}
}
\makeatother

\usepackage[T1]{fontenc}
\arxiv{\usepackage{lmodern}}  %
\usepackage[scaled=.91]{helvet}
\usepackage{inconsolata}
\usepackage{microtype}
\usepackage{bm}
\usepackage{float}

\allowdisplaybreaks
\arxiv{\emergencystretch=2em}

\title{Anatomy of Off-Policy Policy Gradient: Importance Sampling, KL Regularization, and Baselines}

\author{%
  Haoqun Cao\quad Yurun Yuan\quad Tengyang Xie\\
  University of Wisconsin--Madison\\
  \texttt{hcao65@wisc.edu, yurun\_yuan@cs.wisc.edu, tx@cs.wisc.edu}
}

\date{}

\begin{document}

\maketitle

\begin{abstract}
Reinforcement learning has become a central tool for large language model (LLM) post-training, where policy gradient methods are routinely deployed off-policy, even though vanilla policy gradient assumes on-policy sampling. We study when such off-policy deployment is justified, and what roles its core ingredients, including importance sampling, KL regularization, and baselines, play in handling the resulting distribution shift. We show that for a general class of off-policy policy gradient updates, consistency, in the sense that the optimal policy remains a stationary point, requires one of two structures at the optimal policy: trajectory-level importance weighting, which suffers from the well-known curse of horizon, or KL regularization with a baseline equal to the optimal soft value, instantiated by the grouped mean. Focusing on the KL route, we identify it with the group-relative policy gradient (GRPG) update, establish an equivalence to a trajectory-level Bellman residual minimization objective, and obtain a finite-sample regret bound under all-policy coverage. This also gives a theoretical account of a variant of the empirically successful GRPO objective. We further uncover a hidden role of baselines in the off-policy setting: a constant shift in the baseline implements pessimism, the standard remedy in offline RL for partial data coverage. We instantiate this via the asymmetric REINFORCE objective, and connect it to the game-theoretic offline RL framework, showing a regret bound with single-policy coverage.  Together, these results offer a unified theoretical view of off-policy policy gradient in LLM post-training and connect it to classical ideas in offline RL.
\end{abstract}

\section{Introduction}

The recent years have witnessed a tremendous amount of work using reinforcement learning (RL) for large language models (LLMs), with applications ranging from instruction following and human preference alignment \txreplace{\citep{rafailov2023direct}}{\citep{ziegler2019fine,ouyang2022training,bai2022training}} to mathematical and code reasoning \citep{shao2024deepseekmath,guo2025deepseek}. The task of next-token prediction \txedit{is }formulated as  a long-horizon sequential decision problem, and the language model is trained as a policy through standard policy optimization techniques.

Among policy optimization methods in RL, policy gradient \citep{williams1992simple, sutton1999policy} is a fundamental approach and a building block of many modern algorithms. It has also been widely adopted in LLM post-training \citep{ahmadian2024back}. In its vanilla form, policy gradient requires on-policy sampling to construct an unbiased gradient estimate; in practice, however, it is routinely deployed off-policy, with various modifications, due to considerations such as replay buffers \citep{wang2017sample}, distributed actor-learner asynchrony \citep{espeholt2018impala}, or the cost of regenerating rollouts after each update \citep{schulman2017proximal, shao2024deepseekmath}. \textbf{A fundamental question is then: when does off-policy policy gradient work for LLM post-training, and how do its core ingredients --- importance sampling, KL regularization, and baselines --- play their roles?} To our knowledge, this question has not been systematically studied in the LLM literature, and this work aims to fill the gap.

When data is sampled from a behavior policy rather than the current learner, one inherits the setting of offline (or batch) RL \citep{lange2012batch, fujimoto2019off,levine2020offline}, where a long line of work has converged on the \emph{pessimism principle}: under partial data coverage, the learner should be conservative on regions where data is sparse, in order to avoid overestimation \citep{jin2021pessimism,xie2021bellman}. \textbf{A second focus of this work is to bridge the empirical practice of off-policy LLM post-training with the principles developed in the offline RL literature.}

Classical off-policy policy gradient methods make the update valid by recovering a quantity of the target policy from behavior data: expectations under $\pi$ via importance weighting \citep{precup2000eligibility,jiang2016doubly,wang2017sample,espeholt2018impala}, the state distribution $d^\pi$ via density-ratio estimation \citep{liu2019off}, or the action value $Q^\pi$ via an off-policy critic \citep{degris2012off}. None of these scales comfortably to LLM post-training: trajectory-level importance weights have variance exponential in the horizon, density ratios over text states are hard to estimate, and a critic adds substantial overhead. The off-policy objectives that work in practice use none of them, relying only on KL regularization and a grouped-mean baseline; this paper asks why that minimal set of ingredients suffices (see \cref{app:related_work} for further discussion).

\txedit{The off-policy objectives we study \citep[e.g.,][]{shao2024deepseekmath,yuan2025trajectory,arnal2025asymmetric} and related works in this direction with similar objectives \citep[e.g.,][]{ji2024enhancing,wang2025offline,tang2025rl,wang2026off} have been extensively studied empirically in recent LLM post-training work; our goal in this paper is to provide a unified theoretical understanding of when and why they remain consistent under off-policy distribution shift, what role each ingredient (importance sampling, KL, baseline) plays in this, and how the resulting guarantees connect to the offline RL literature.}

\paragraph{Contributions}
This work develops a theoretical framework for off-policy policy gradient in LLM post-training.

\textit{(1) A characterization of consistent off-policy PG updates.} 
We show that for a general class of off-policy PG updates in the autoregressive setting, a consistent update, i.e., one that preserves the optimal policy as a stationary point, must at the optimal policy either use \emph{trajectory-level importance weights} or use  \emph{KL regularization paired with a baseline equal to the optimal soft value} (Proposition~\ref{prop:neccesary_condition_correcting_oppg}). The importance-sampling route suffers from the well-known curse of horizon, so we focus on the KL route, with the grouped-mean baseline, in the rest of the paper.

\textit{(2) Analysis of KL regularization with a grouped-mean baseline.} We \txreplace{instantiate}{identify} the KL route \txreplace{as}{with} the \emph{group-relative policy gradient} (GRPG) update, develop the intuition behind its correctness, and establish an equivalence to a trajectory-level Bellman residual minimization \txreplace{objective}{(TBRM) objective of \citet{yuan2025trajectory}} (\cref{prop:grpg_tbrm_equivalence}). This connects GRPG\txreplace{ or the similarempirically}{, a variant of the empirically} successful GRPO objective \citep{shao2024deepseekmath}\txreplace{ -to}{, to} classical value-based RL principles. We further give a finite-sample regret bound under all-policy coverage (\cref{thm:bound_all_policy_coverage}).

\textit{(3) A hidden role of baselines: implementing pessimism.} Within the KL route, we show that a constant shift in the baseline has a surprising effect beyond variance reduction: it implements \emph{pessimism}, the standard remedy in offline RL for partial coverage. \txreplace{We instantiate this via the asymmetric REINFORCE (AsymRE) objective \citep{arnal2025asymmetric}, give it}{We connect this insight to the asymmetric REINFORCE (AsymRE) objective of \citet{arnal2025asymmetric}, providing} a game-theoretic interpretation connecting it to the adversarial offline RL framework of \citet{cheng2022adversarially}, and upgrade the regret bound from all-policy to single-policy coverage (\cref{thm:bound_single_policy_coverage}).

Taken together, these three parts trace a single path from the off-policy policy-gradient update template to the two finite-sample guarantees. \cref{fig:roadmap} lays out that path, together with the equivalences and the coverage conditions encountered along the way, and may be useful as a map while reading the rest of the paper.

\begin{figure}[t]
    \centering
    \includegraphics[width=\linewidth]{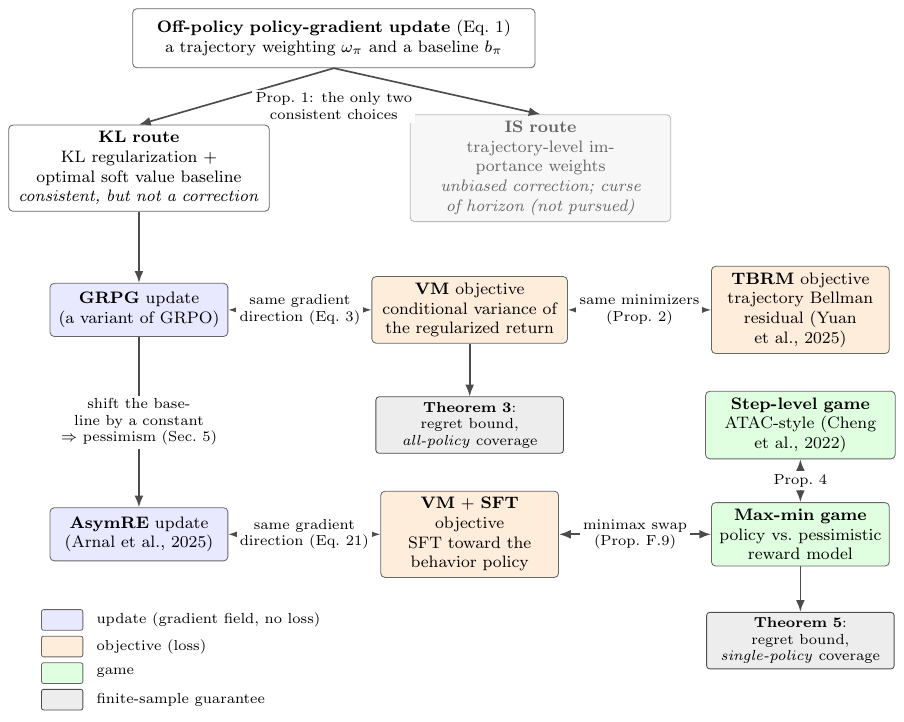}
    \caption{Roadmap of the paper: the off-policy policy-gradient update template, the two consistent routes it admits, the chain of equivalent objectives and games along the KL route, and the two finite-sample guarantees.}
    \label{fig:roadmap}
\end{figure}

\section{Preliminaries}
This section provides the necessary background for our work. We first review the fundamentals of episodic Markov Decision Process (MDP) and KL-regularized reinforcement learning, then discuss the policy gradient methods.

\subsection{Episodic MDP and KL-regularized RL}

\paragraph{MDP Fundamentals}
We consider a finite-horizon MDP represented by $M=(H, \mathcal S,\mathcal A, T, r^*,\rho)$ consisting of a state space $\mathcal S$, an action space $\mathcal A$, a reward function $r^*=\{r^*_h\}, r^*_h:\mathcal S\times \mathcal A\rightarrow [0,1]$, an initial state distribution $\rho\in \Delta(\mathcal S)$, and a \textbf{deterministic transition} function $T=\{T_h\}, T_h:\mathcal S\times \mathcal A\rightarrow \mathcal S$, which we assume throughout the paper and which is central to our analysis. A policy $\pi=\{\pi_h\},\pi_h:\mathcal S\rightarrow \Delta(\mathcal A)$ maps states to distributions over
actions and induces a distribution over trajectories $\tau=(s_1,a_1,\ldots, s_H, a_H)$, where we write $s_h(\tau)$ and $a_h(\tau)$ for the state and action of $\tau$ at step $h$; the conditional distribution of $\tau$ given $s_1$ induced by policy $\pi$ is denoted as $\mathbb P^\pi(\cdot|s_1)$. We also define the trajectory level reward function as $ R^*(\tau):=\sum_{h=1}^H r^*_h(s_h,a_h)$. We use $\mathbb E_{\pi}[\cdot]$ to denote the expectation under such a data generating process induced by $\pi$: $s_1\sim \rho, a_h\sim \pi_h(\cdot|s_h), s_{h+1}=T_h(s_h,a_h), h=1,\ldots H$, or equivalently $s_1\sim\rho, \tau\sim \mathbb P^\pi(\cdot|s_1)$. For any policy $\pi$, we define  $d^\pi_h(\cdot, \cdot)\in \Delta(\mathcal S\times \mathcal A)$ as the induced marginal state-action distribution at step $h$. The goal of RL is to find a policy $\pi$ that maximizes the expected cumulative reward $J(\pi):=\mathbb E_{\pi}[R^*(\tau)]$.

\paragraph{KL-regularized RL}
In many applications, a KL regularization term is incorporated to encourage the learned policy to stay close to a reference policy $\pi_{\textnormal{ref}}$. This gives rise to the KL-regularized return and objective:
\begin{align*}
    R^*_{\pi,\beta}(\tau)
    := R^*(\tau)-\beta \sum_{h=1}^H \log\frac{\pi_h(a_h|s_h)}{\pi_{\textnormal{ref},h}(a_h|s_h)},
    \qquad
    J_\beta(\pi):=\mathbb E_{\tau\sim\pi}\big[R^*_{\pi,\beta}(\tau)\big].
\end{align*}
Throughout the paper, we let $\pi^*$ denote the optimal policy of $J(\pi)$ in the unregularized case and of $J_\beta(\pi)$ in the KL-regularized case; the relevant case will be clear from context.

\paragraph{Soft value functions}
We now define the soft value functions via backward recursion under the deterministic transition $s_{h+1}=T_h(s_h,a_h)$. Set $V^\pi_{\beta,H+1}(s):=0$. For $h=H,H-1,\ldots,1$, define
\begin{align*}
    Q^\pi_{\beta,h}(s,a)
    &:=
    r_h^*(s,a)+V^\pi_{\beta,h+1}\big(T_h(s,a)\big),\\
    V^\pi_{\beta,h}(s)
    &:=
    \mathbb E_{a\sim \pi_h(\cdot|s)}\!\left[
        Q^\pi_{\beta,h}(s,a)
        -\beta\log\frac{\pi_h(a|s)}{\pi_{\textnormal{ref},h}(a|s)}
    \right].
\end{align*}
We write $V^*_{\beta,h}, Q^*_{\beta,h}$ for the soft value functions of the optimal policy. In the unregularized case, we only need the optimal value at the initial state, $V^*_1(s_1):=\max_{\tau}R^*(\tau)$, where the maximum is over trajectories from $s_1$.

\paragraph{Coverage}
We will also discuss \emph{coverage conditions}, which measure how well the offline data distribution covers the distributions induced by the learner. The strongest such requirement, \emph{all-policy coverage}, asks that the data cover every policy in the policy class; this is typically needed when the algorithm has no built-in protection against overestimation. A weaker and more refined notion is \emph{single-policy coverage}: for any comparator policy $\pi$ that is well covered by the data, the algorithm is guaranteed to perform at least as well as $\pi$. Throughout the paper, the specific coverage condition required by each algorithm is stated in its assumptions, typically as a bound on the density ratio between $d^\pi$ (or $d^{\pi^*}$) and $d^\mu$, or on the $\chi^2$-divergence between them, denoted $\chi^2(P\|Q)$.

\subsection{Policy Gradient Method}
Throughout, policies are parameterized, $\pi=\pi_\theta$ with $\theta\in\Theta$, and $\nabla$ denotes the gradient with respect to $\theta$; we suppress $\theta$ when it is clear from context. Policy gradient methods are motivated by the well-known policy gradient lemma, which directly expresses the gradient of $J(\pi)$ in terms of $\pi$:
\begin{equation*}
    \nabla J(\pi)=\mathbb E_\pi\left[R^*(\tau)\sum_{h=1}^H \nabla \log \pi_h(a_h|s_h)\right].
\end{equation*}
Computing this gradient requires sampling from the current policy (i.e., on-policy sampling) and running stochastic gradient descent. In this paper, we instead consider the off-policy or offline setting, where data is collected by a fixed behavior policy $\mu$. We say that a policy $\pi$ is \emph{supported on $\pi_{\textnormal{ref}}$} if $\textnormal{supp}(\pi_h(\cdot|s))\subseteq \textnormal{supp}(\pi_{\textnormal{ref},h}(\cdot|s))$ for all $(s,h)$, and assume throughout that $\mu$ is supported on $\pi_{\textnormal{ref}}$. In this case, simply replacing the measure $\pi$ with $\mu$ in the above expectation no longer yields $\nabla J(\pi)$.

\section{Two Consistent Routes for Off-Policy Policy Gradient}
\label{sec:correct-off-policy}
In the off-policy case, the naive substitution therefore has to be modified. In this work, we consider off-policy updates of the following form:
\begin{equation}\label{eq:off_policy_PG_template}
    \nabla\mathcal E_{\mu,\textrm{PG}}(\pi):=
    -\mathbb E_{\tau\sim \mu}\left[
    \omega_\pi(\tau) A_\pi(\tau)
    \sum_{h=1}^H \nabla \log \pi_h(a_h|s_h)
    \right],
\end{equation}
where
\begin{align*}
    A_\pi(\tau)
    :=&R^*(\tau)-b_\pi(s_1),
    \qquad \textrm{without KL regularization},\\
    A_\pi(\tau)
    :=&R^*_{\pi,\beta}(\tau)-b_\pi(s_1),
    \qquad \textrm{with KL regularization}.
\end{align*}
Here $b_\pi$ is a policy-dependent baseline depending only on the initial state $s_1$, and $\omega_\pi$ is a policy-dependent weighting function defined over the entire trajectory. With a slight abuse of notation, $\nabla\mathcal E_{\mu,\textrm{PG}}(\pi)$ denotes the update direction itself and need not be the gradient of a function of $\pi$; the same convention applies to the updates introduced below. This formulation captures a broad class of off-policy policy-gradient objectives and incorporates the key ingredients commonly used in practice to mitigate off-policy distribution shift:
\begin{enumerate}
    \item Importance weighting, represented by $\omega_\pi(\tau)$.
    \item KL regularization.
    \item A baseline, represented by $b_\pi(s_1)$.
\end{enumerate}
We will show that, for $\pi^*$ to be a stationary point of this update, $\omega_\pi$ and $A_\pi$ must take particular forms.

We say that $\pi$ is a \emph{stationary point} of the update \cref{eq:off_policy_PG_template} if for every  $h\in[H]$, $s\in\mathcal S$, and $a\in\mathcal A$,
\begin{equation}\label{eq:stationarity_condition}
    \mathbb E_{\tau\sim\mu}\Big[\omega_\pi(\tau)A_\pi(\tau)\Big(\mathbf 1\{s_h(\tau)=s,\ a_h(\tau)=a\}-\mathbf 1\{s_h(\tau)=s\}\pi_h(a|s)\Big)\Big]=0.
\end{equation}
Under the tabular softmax parameterization $\pi_{\theta,h}(a|s)=\exp(\theta_h(s,a))/\sum_{a'}\exp(\theta_h(s,a'))$, we have $\frac{\partial}{\partial\theta_h(s,a)}\log\mathbb P^{\pi_\theta}(\tau|s_1)=\mathbf 1\{s_h(\tau)=s,\ a_h(\tau)=a\}-\mathbf 1\{s_h(\tau)=s\}\pi_{\theta,h}(a|s)$, so \cref{eq:stationarity_condition} is the componentwise form of $\nabla\mathcal E_{\mu,\textnormal{PG}}(\pi_\theta)=0$ with respect to $\theta$; unlike the latter, it remains well defined for policies on the boundary of the simplex, such as the optimal policy without KL regularization.

We say that the update \cref{eq:off_policy_PG_template} is \emph{consistent} if $\pi^*$ is a stationary point of it, i.e., if the update preserves stationarity of the optimal policy. Consistency is weaker than unbiasedness: a consistent update need not recover the on-policy gradient $\nabla J(\pi)$ at other policies. The resulting forms of $\omega_\pi$ and $A_\pi$ give rise to two provably consistent routes for off-policy policy optimization: one driven by importance sampling (IS) weights, and the other by KL regularization.
\begin{proposition}\label{prop:neccesary_condition_correcting_oppg}
Consider the autoregressive setting $s_h=(s_1,a_1,\ldots,a_{h-1})$ with discrete initial-state and action spaces. If $\pi^*$ satisfies \cref{eq:stationarity_condition}, then for every $s_1\in\textnormal{supp}(\rho)$:
    \begin{enumerate}
        \item Without KL regularization, there exists $K(s_1)$ such that
        \begin{equation*}
            \omega_{\pi^*}(\tau)=K(s_1)\prod_{h=1}^H \frac{\pi_h^*(a_h|s_h)}{\mu_h(a_h|s_h)}
        \end{equation*}
        for every $\tau\in\textnormal{supp}(\mathbb P^\mu(\cdot|s_1))$ with $R^*(\tau)\neq b_{\pi^*}(s_1)$.
        \item With KL regularization, at least one of the following holds: there exists $K(s_1)$ such that
        \begin{equation*}
            \omega_{\pi^*}(\tau)=K(s_1)\prod_{h=1}^H \frac{\pi_h^*(a_h|s_h)}{\mu_h(a_h|s_h)}\quad\text{for every }\tau\in\textnormal{supp}(\mathbb P^\mu(\cdot|s_1)),
        \end{equation*}
        or
        \begin{equation*}
            b_{\pi^*}(s_1)=V^*_{\beta,1}(s_1)=\mathbb E_{\mu}\left[R^*_{\pi^*,\beta}(\tau) |s_1\right].
        \end{equation*}
    \end{enumerate}
\end{proposition}
See proof in \cref{proof_of_two_paths_prop}. The above proposition shows that there are only two candidates at $\pi^*$ for a consistent update: either through trajectory-level importance weights or through KL regularization (with baseline as the soft optimal value). In particular, trajectory-level importance weight has the potential to address the off-policy PG objective, both with and without KL regularization. And under KL regularization, there is an alternative, e.g., by using conditional mean as the baseline (which is the most natural choice and which we adopt throughout).  Note that this proposition only states \textbf{necessary conditions}; we will show subsequently that both choices are indeed consistent. At the same time, the necessity also implies that many alternative choices, or arbitrary combinations of these ingredients, may fail: in such cases, $\pi^*$ may not even be a stationary point of the resulting objective. The two routes differ in nature. Trajectory-level importance weighting is a \emph{correction} in the classical sense: it recovers the on-policy gradient in expectation, and its sufficiency is well known \citep{precup2000eligibility,jiang2016doubly}; the content of \cref{prop:neccesary_condition_correcting_oppg} is the converse. The KL route recovers no on-policy quantity; it is consistent because it changes the optimality condition so that the advantage vanishes at $\pi^*$, as we explain next. Now, we talk about the first route through trajectory-level importance weights.

\paragraph{The IS Route: Curse of Horizon}
Setting $\omega_\pi(\tau)=\prod_{h=1}^H \pi_h(a_h|s_h)/\mu_h(a_h|s_h)$ yields the classical unbiased correction: the update equals $-\nabla J(\pi)$ (resp.\ $-\nabla J_\beta(\pi)$) at the population level. However, since the weight is a product of $H$ per-step policy ratios, its variance typically grows exponentially in the horizon $H$ — the well-known \emph{curse of horizon} \citep{jiang2016doubly,wang2017sample,liu2019off}. This makes the trajectory-level IS route impractical in long-horizon settings, including LLM training. We therefore do not pursue the IS route further. 
Instead, we focus on the KL route, which we develop in the next section.

\section{The KL Route: Grouped Mean Baseline}
\label{sec:kl-route}

KL regularization is one of the hidden gems for off-policy policy optimization. It provides a different route to consistency: unlike trajectory-level importance weighting, it does not recover the on-policy gradient, and so is not a correction in the classical sense; instead, it changes the optimality condition itself. We now explain how KL regularization enables a properly chosen advantage function to satisfy $A_{\pi^*}(\tau)=0$ at the optimum, thus making $\pi^*$ a stationary point.

Fix an initial state $s_1$, without KL regularization, the objective is linear in the trajectory distribution:
\begin{equation*}
    \max_{\pi}
    \mathbb E_{\tau\sim\mathbb P^\pi(\cdot|s_1)}[R^*(\tau)].
\end{equation*}
Hence, the optimum typically lies on the boundary of the probability simplex, and the optimality condition is an inequality: $R^*(\tau)\leq V_1^*(s_1)$, with equality only on selected optimal trajectories. Thus, unless all rewards are tied, the raw advantage cannot be zero for every trajectory.

With KL regularization, the objective becomes
\begin{equation*}
    \max_{\pi}
    \mathbb E_{\tau\sim\mathbb P^\pi(\cdot|s_1)}
    \left[
    R^*(\tau)-\beta\log
    \frac{\mathbb P^\pi(\tau|s_1)}
    {\mathbb P^{\pi_{\textnormal{ref}}}(\tau|s_1)}
    \right].
\end{equation*}
The KL term makes the objective strictly concave in the trajectory distribution and forces the optimum to have full support over $\pi_{\textnormal{ref}}$. As a result, the optimality condition becomes a pointwise equality: for every trajectory $\tau$ in the support of $\mathbb P^{\pi_{\textnormal{ref}}}$,
\begin{equation*}
    R^*(\tau)-\beta\log
    \frac{\mathbb P^{\pi^*}(\tau|s_1)}
    {\mathbb P^{\pi_{\textnormal{ref}}}(\tau|s_1)}
    =V^*_{\beta,1}(s_1)=\mathbb E_\mu\!\left[R^*_{\pi^*,\beta}(\tau)\,\big|\,s_1\right].
\end{equation*}
A derivation of this identity is given in \cref{sec:revisit_KL_RL}. This means the KL-regularized return $R^*_{\pi^*,\beta}(\tau)$ is constant along every trajectory in the reference support. Consequently, choosing the grouped mean baseline $b_{\pi}(s_1)=\mathbb E_\mu[R^*_{\pi,\beta}(\tau)|s_1]$ yields $A_{\pi^*}(\tau)=0$ for every $\tau$ in the support, which makes $\pi^*$ a stationary point of the off-policy policy-gradient objective.
\begin{remark}
    The argument relies on the optimality condition holding along every trajectory, which is where \textbf{deterministic transitions} are used. Under stochastic transitions it holds only in expectation over the next state, so the advantage does not vanish pathwise and $\pi^*$ is in general not a stationary point of the KL route. The same trajectory-level structure underlies previous literature \citep{nachum2017bridging,yuan2025trajectory}.
\end{remark}

\subsection{Value-Based RL interpretation}
Surprisingly, we show that KL regularization, together with the grouped mean baseline, also admits an equivalent value-based RL interpretation. The off-policy PG update we study is \cref{eq:off_policy_PG_template} with $\omega_\pi\equiv 1$ and the grouped-mean baseline,
\begin{equation*}
    \nabla\mathcal E_{\mu, \textnormal{GRPG}}(\pi):=-\mathbb E_{\mu}\left[\left(R^*_{\pi,\beta}(\tau)-\mathbb E_\mu[R^*_{\pi,\beta}(\tau)\mid s_1]\right)\sum_{h=1}^H  \nabla\log \pi_h(a_h|s_h)\right],
\end{equation*}
where GRPG abbreviates \emph{group-relative policy gradient}.

Building on the equivalence framework of soft Q learning and REINFORCE in \citet{schulman2017equivalence}, we first connect the GRPG update to the following conditional variance minimization (VM) objective,
\begin{align*}
    \mathcal E_{\mu,\textnormal{VM}}(\pi):=&\mathbb E_{s_1\sim\rho}\left[\textnormal{Var}_\mu\left(R^*_{\pi,\beta}(\tau)\big |s_1\right)\right].
\end{align*}
The conditional variance objective itself was previously studied as a standalone loss in the contextual-bandit setting by \citet{tang2025rl}; here we show that its gradient coincides with the GRPG update up to a positive constant. A direct calculation shows
\begin{equation}\label{eq:PG_TO_VM}
    \nabla  \mathcal E_{\mu,\textnormal{VM}}(\pi)= 2\beta\nabla\mathcal E_{\mu,\textnormal{GRPG}}(\pi).
\end{equation}
Therefore, following the GRPG update is the same as running gradient descent on $\mathcal E_{\mu,\textnormal{VM}}$, up to the positive scaling constant $2\beta$; the same holds for the corresponding empirical estimators. We now connect $\mathcal E_{\mu,\textnormal{VM}}(\pi)$ to value-based RL theory and show that it shares its minimizing policy with a trajectory-level Bellman residual minimization objective.

Let us consider a $Q$ function class $\mathcal F\subset \{f_h:\mathcal S\times \mathcal A\rightarrow \mathbb R\}$ and the following trajectory Bellman residual minimization (TBRM, originally from \citet{yuan2025trajectory}) objective,
\begin{equation*}
    \mathcal E_{\mu,\textnormal{TBRM}}(f)
    :=
    \mathbb E_{\tau\sim\mu}\left[
    \left(\sum_{h=1}^H \big(f_h(s_h,a_h)-(\mathcal T_{\beta,h}^* f)(s_h,a_h)\big)\right)^{2}
    \right],
\end{equation*}
where $\mathcal T^*_{\beta}$ is the soft Bellman optimality operator induced by the KL-regularized MDP. Under the deterministic transition $s_{h+1}=T_h(s_h,a_h)$, it admits the closed form
\begin{equation*}
    (\mathcal T_{\beta,h}^* f)(s_h,a_h)
    =
    r_h^*(s_h,a_h)
    +
    \beta\log \!\int_{\mathcal A}\!\pi_{\textnormal{ref},h+1}(a|s_{h+1})\exp\!\left(\tfrac{f_{h+1}(s_{h+1},a)}{\beta}\right)\!da,
\end{equation*}
with the second term set to $0$ at $h=H$. After an $f\in\mathcal F$ is learned, the corresponding policy is the soft-greedy policy $\pi_f$ with respect to $f$,
\begin{equation*}
    \pi_{f,h}(a|s)
    \propto
    \pi_{\textnormal{ref},h}(a|s)\exp\!\left(f_h(s,a)/\beta\right).
\end{equation*}
Detailed derivations of $\mathcal T^*_\beta$ and $\pi_{f}$ are deferred to \cref{sec:revisit_KL_RL}. We now state the equivalence between VM and TBRM.
\begin{proposition}\label{prop:grpg_tbrm_equivalence}
    Let $\mathcal F$ be a value function class that is closed under first-step, state-dependent, action-independent shifts, in the sense that for any $f\in\mathcal F$ and any $c:\mathcal S\to\mathbb R$ depending only on $s_1$, the function $f^c$ defined by $f^c_1(s_1,a_1)=f_1(s_1,a_1)+c(s_1)$ and $f^c_h=f_h$ for all $h\ge 2$ also belongs to $\mathcal F$. Denote $\Pi_{\mathcal F}:=\{\pi_f : f\in \mathcal F\}$. Then
    \begin{align*}
        \argmin_{\pi\in \Pi_{\mathcal F}}\mathcal E_{\mu,\textnormal{VM}}(\pi)
        =
        \left\{\pi_{f^*} : f^*\in\argmin_{f\in\mathcal F} \mathcal E_{\mu,\textnormal{TBRM}}(f)\right\}.
    \end{align*}
\end{proposition}
The intuition is that the TBRM objective contains a free degree of freedom that does not affect the induced policy: shifting any $f$ by a function of $s_1$ leaves $\pi_f$ unchanged but allows the squared Bellman residual to be reduced. When $\mathcal F$ is closed under such shifts, the optimal shift is precisely the conditional mean over $s_1$, which turns the squared error into the conditional variance and yields the equivalence. Therefore, the grouped-mean off-policy policy-gradient update can also be viewed as minimizing a trajectory-level Bellman residual, providing a value-based interpretation of the method and connecting it to Bellman-residual and value-iteration principles in reinforcement learning.

\subsection{Finite Sample Analysis}
Finally, we provide a finite-sample analysis of the GRPG update, showing that the KL route with the grouped-mean baseline provably learns the optimal policy under a coverage condition. We also discuss a potential limitation of this algorithm from the perspective of offline RL. Let
$D=\{s_1^i, a_{1:H}^{ij}, s_{2:H+1}^{ij}\}_{i=1,\,j=1}^{i=n,\,j=m}$,
where $s_1^i\sim \rho$ i.i.d.\ and, conditioning on $s_1^i$, $\tau^{ij}\sim \mathbb P^\mu(\cdot|s_1^i)$ i.i.d.\ across $j$. Define the empirical GRPG update
\begin{equation*}
\begin{aligned}
    \nabla\mathcal E_{D,\textnormal{GRPG}}(\pi)
    :=
    -\frac{1}{n}\sum_{i=1}^n \frac{1}{m-1}\sum_{j=1}^m
    &\left(
        R^*_{\pi,\beta}(\tau^{ij})
        -\frac{1}{m}\sum_{k=1}^m R^*_{\pi,\beta}(\tau^{ik})
    \right)\\
    &\times\left(\sum_{h=1}^H \nabla\log \pi_h(a_h^{ij}|s_h^{ij})\right).
\end{aligned}
\end{equation*}
Notice that this is a variant of the gradient of the GRPO objective in \citet{shao2024deepseekmath}, which differs in details such as the importance ratio, the normalization, and how the KL term enters. The empirical counterpart of $\mathcal E_{\mu,\textnormal{VM}}$ is
\begin{equation}\label{eq:empirical_VM}
    \mathcal E_{D,\textnormal{VM}}(\pi):=\frac{1}{n}\sum_{i=1}^n\frac{1}{m-1}\sum_{j=1}^m\left(R^*_{\pi,\beta}(\tau^{ij})-\frac{1}{m}\sum_{k=1}^m R^*_{\pi,\beta}(\tau^{ik})\right)^2,
\end{equation}
and a direct computation (\cref{eq:propto_gra_GRPG_VM}) gives $\nabla\mathcal E_{D,\textnormal{VM}}(\pi)=2\beta\nabla\mathcal E_{D,\textnormal{GRPG}}(\pi)$, so the GRPG update is the gradient field of $\mathcal E_{D,\textnormal{VM}}$. We therefore analyze its minimizer. Let $\Pi\subset \{\pi_h:\mathcal S\to \Delta(\mathcal A)\}$ be a policy class, and define
\begin{equation*}
    \hat \pi\in \argmin_{\pi\in\Pi}\mathcal E_{D,\textnormal{VM}}(\pi).
\end{equation*}
We make the following assumptions on data coverage and boundedness.
\begin{assumption}\label{assume:all_policy_state_action_coverage}
There exist constants $C_{\chi^2}, B<\infty$ such that
\begin{equation*}
\begin{aligned}
    \text{\emph{(All-policy coverage)}}\quad
    &\sup_{\pi\in \Pi}\sum_{h=1}^H\chi^2\left(d^\pi_h\big\|d^{\mu}_h\right)\leq C_{\chi^2},\\
    \text{\emph{(Bounded trajectory log-ratio)}}\quad
    &\sup_{\pi\in\Pi,\ s_1,\ \tau\in\textnormal{supp}(\mathbb P^{\pi_{\textnormal{ref}}}(\cdot|s_1))}\left|\beta\log\frac{\mathbb P^\pi(\tau|s_1)}{\mathbb P^{\pi_{\textnormal{ref}}}(\tau|s_1)}\right|\leq B.
\end{aligned}
\end{equation*}
\end{assumption}
We then obtain the following guarantee.
\begin{theorem}\label{thm:bound_all_policy_coverage}
    Suppose \cref{assume:all_policy_state_action_coverage} holds, $\pi^*\in \Pi$, and $|\Pi|<\infty$. Then with probability at least $1-\zeta$,
    \begin{equation*}
        J_\beta(\pi^*)-J_\beta(\hat \pi)
        \lesssim
        (H+B)\sqrt{\frac{(1+C_{\chi^2})\log(|\Pi|/\zeta)}{n}}.
    \end{equation*}
\end{theorem}
\begin{proof}[Proof sketch]
By a soft performance-difference identity (\cref{lemma:reward_to_KL}), $J_\beta(\pi^*)-J_\beta(\hat\pi)=\beta\,\mathbb E_{s_1\sim\rho}\big[\textnormal{KL}\big(\mathbb P^{\hat\pi}(\cdot|s_1)\,\|\,\mathbb P^{\pi^*}(\cdot|s_1)\big)\big]$. A change of measure from $\hat\pi$ and $\pi^*$ to $\mu$ (\cref{lemma:KL_to_variance}), which is where all-policy coverage enters, bounds this expected KL by $\sqrt{(1+C_{\chi^2})\,\mathcal E_{\mu,\textnormal{VM}}(\hat\pi)}$ up to constants, because $R^*_{\pi,\beta}(\tau)$ and $-\beta\log\frac{\mathbb P^\pi(\tau|s_1)}{\mathbb P^{\pi^*}(\tau|s_1)}$ differ only by a function of $s_1$ and hence have the same conditional variance. Finally, since $R^*_{\pi^*,\beta}$ is constant on the support, $\mathcal E_{D,\textnormal{VM}}(\pi^*)=0$ pathwise, and a Bernstein-type excess-risk bound (\cref{lemma:excess_risk_concentration}) gives $\mathcal E_{\mu,\textnormal{VM}}(\hat\pi)\lesssim (H+B)^2\log(|\Pi|/\zeta)/n$. See \cref{sec:proof_thm_bound_all_policy}.
\end{proof}

\begin{remark}
The first condition of \cref{assume:all_policy_state_action_coverage} is the standard \textbf{all-policy coverage} assumption in offline RL: the offline data must cover the distributions induced by the entire policy class, not just $\pi^*$. Such a strong requirement is typical for algorithms with no built-in pessimism against overestimation; in the next section, we show that a constant shift in the baseline implements exactly this remedy.
\end{remark}

\begin{remark}
The second condition of \cref{assume:all_policy_state_action_coverage} is needed to simplify the concentration analysis: it bounds the KL-regularized return, $|R^*_{\pi,\beta}(\tau)|\leq H+B$, in \cref{lemma:excess_risk_concentration}. It does not conflict with realizability: by \cref{eq:trajectory_reward_KL_optimality}, $\beta\log\frac{\mathbb P^{\pi^*}(\tau|s_1)}{\mathbb P^{\pi_{\textnormal{ref}}}(\tau|s_1)}=R^*(\tau)-V^*_{\beta,1}(s_1)\in[-H,H]$ since both terms lie in $[0,H]$, so $\pi^*$ automatically satisfies the condition with $B=H$. Hence, for any policy class $\Pi$ containing $\pi^*$, the subclass of policies satisfying the condition with $B=H$ still contains $\pi^*$, and \cref{thm:bound_all_policy_coverage} applies to it as is --- although a generic policy class used in LLM training does not automatically satisfy the condition. A condition of the same form also appears in \cref{thm:bound_single_policy_coverage}.
\end{remark}

\section{A Hidden Role of Baseline in Off-Policy Policy Gradient}
\label{sec:baseline-const-shift}
In on-policy policy gradient, a baseline is a control variate that reduces gradient variance without changing the population gradient \citep{williams1992simple,peters2008reinforcement}. In the off-policy setting the baseline plays a different role: as we have seen in \cref{prop:neccesary_condition_correcting_oppg}, along the KL route it is no longer free to choose: its value at $\pi^*$ must equal the optimal soft value, which the grouped mean achieves. In this section, we show that the baseline carries an even more interesting role beyond ensuring stationarity: a properly shifted baseline implicitly induces \emph{pessimism}, the standard remedy in offline RL to mitigate overestimation with partial coverage.

\txreplace{Concretely, consider adding a constant shift to the GRPG objective:}{Concretely, the asymmetric REINFORCE (AsymRE) update of \citet{arnal2025asymmetric} adds a constant shift to the grouped-mean baseline of GRPG:}
\begin{equation*}
    \nabla\mathcal E_{\mu,\textnormal{AsymRE}}(\pi):=-\mathbb E_{\mu}\left[\left(R^*_{\pi,\beta}(\tau)-\left( \mathbb E_\mu[R^*_{\pi,\beta}(\tau)| s_1]-\delta\right)\right)\sum_{h=1}^H  \nabla\log \pi_h(a_h|s_h)\right],
\end{equation*}
where $\delta>0$. \citet{arnal2025asymmetric} study this update with the unregularized return $R^*(\tau)$ and a constant baseline, and discuss KL regularization only as an extension. The KL term is essential for our analysis, since without it the grouped-mean baseline is not consistent (\cref{prop:neccesary_condition_correcting_oppg}) and the VM identity \cref{eq:PG_TO_VM} does not hold. We nevertheless keep the name AsymRE for the KL-regularized form above, since the constant baseline shift is the feature we study and introducing a new name would only add notation. \txreplace{This objective is motivated by asymmetric REINFORCE \citep{arnal2025asymmetric}. Intuitively,}{Intuitively,} subtracting an extra $\delta$ amplifies the weight on high-reward trajectories and weakens the push to suppress low-reward ones, since failures sampled from a behavior policy $\mu\neq\pi$ are not particularly worth learning from. Beyond this heuristic, we show below that the shift admits a principled interpretation: it is equivalent to adversarially choosing a value function or reward model that incorporates pessimism.

Notice that the AsymRE update can be decomposed as
\begin{align*}
    \nabla\mathcal E_{\mu,\textnormal{AsymRE}}(\pi)
    &=\nabla\mathcal E_{\mu,\textnormal{GRPG}}(\pi)+\delta\cdot \underbrace{\left(-\mathbb E_\mu\left[\sum_{h=1}^H\nabla \log \pi_h(a_h|s_h)\right]\right)}_{:=\nabla \mathcal L_{\mu,\textnormal{SFT}}(\pi)}\\
    &=\frac{1}{2\beta}\nabla \mathcal E_{\mu,\textnormal{VM}}(\pi)+\delta\cdot \nabla \mathcal L_{\mu,\textnormal{SFT}}(\pi),
\end{align*}
where the second equality uses $\nabla\mathcal E_{\mu,\textnormal{VM}}(\pi)=2\beta\nabla\mathcal E_{\mu,\textnormal{GRPG}}(\pi)$ from \cref{eq:PG_TO_VM}. Therefore, the AsymRE update is essentially adding an SFT regularizer (SFT towards the behavior policy $\mu$) that prevents the policy from deviating too far from the coverage of the offline data. Beyond this, following \citet{liu2024provably}, we can show that this SFT term also serves as an adversarial regularizer that prevents the model from choosing trajectories with high reward as estimated by its own reward model.

To further elaborate this point, let us consider $\mathcal R\subset \{R: R(\tau)=\sum_{h=1}^H r_h(s_h,a_h),\ r_h:\mathcal S\times \mathcal A\rightarrow \mathbb R\}$ as a reward class, and its element as a reward model $R$. Recall $\Pi$ is the policy class. We define
\begin{equation*}
    \begin{aligned}
        \mathcal L_\mu(\pi, R)
        :=&\mathbb E_{s_1\sim \rho}\Big[\mathbb E_{\tau\sim \mathbb P^\pi(\cdot|s_1)}[R(\tau)]-\mathbb E_{\tau\sim \mathbb P^\mu(\cdot|s_1)}[R(\tau)]\\
        &\qquad\quad-\beta\,\textnormal{KL}\!\left(\mathbb P^\pi(\cdot|s_1)\,\|\,\mathbb P^{\pi_{\textnormal{ref}}}(\cdot|s_1)\right)+\beta\,\textnormal{KL}\!\left(\mathbb P^\mu(\cdot|s_1)\,\|\,\mathbb P^{\pi_{\textnormal{ref}}}(\cdot|s_1)\right)\Big],\\
        {\mathcal E}_\mu(R)
        :=&\mathbb E_{s_1\sim\rho}\textnormal{Var}_{\tau\sim \mathbb P^\mu(\cdot|s_1)}\!\left[R^*(\tau)-R(\tau)\,\big|\,s_1\right].
    \end{aligned}
\end{equation*}
Consider the following trajectory-level max-min game,
\begin{equation}\label{eq:bandit_like_maxmin_game}
    \pi^\dagger\in \argmax_{\pi\in\Pi}\min_{R\in \mathcal R}\left(\mathcal L_\mu(\pi, R)+\eta\cdot {\mathcal E}_{\mu}(R)\right).
\end{equation}
In \cref{sec:proof_thm_single_policy}, we show that under regularity conditions and with an appropriate choice of $\eta$ (depending on $\delta$ and $m$), optimizing the VM+SFT objective (thus following the AsymRE update) is equivalent to solving this max-min game with $\mathcal R$ taken to be a reward class induced by $\Pi$. The structure of this max-min game reveals why pessimism emerges. In the inner minimization, ${\mathcal E}_\mu$ keeps the reward model consistent with the offline data, while $\mathcal L_\mu(\pi, R)$ adversarially picks the most pessimistic reward model under $\mathbb P^\pi$; the outer maximization then maximizes against this adversarially chosen reward model. In this way, the AsymRE loss overcomes overestimation of the reward, and also leads to a guarantee with single-policy coverage as we will show later.

\paragraph{Connection to Offline RL Literature}
Beyond this trajectory-level contextual bandit-type formulation, the use of game-theoretic objectives for offline pessimism is well rooted in the offline RL literature \citep{xie2021bellman, cheng2022adversarially}. We will show that the AsymRE objective also admits a formulation along these lines. Following \citet{cheng2022adversarially}, for $f\in\mathcal F\subset  \{f_h:\mathcal S\times \mathcal A\rightarrow \mathbb R\}$ and a policy $\pi$, we first define
\begin{align*}
    \tilde{\mathcal L}_\mu(\pi, f)&:=\sum_{h=1}^H \mathbb E_{s_h\sim d^\mu_h}\!\left[V^\pi_{f,h}(s_h)-V^\mu_{f,h}(s_h)\right],\\
    V_{f,h}^\pi(s_h)&:=\mathbb E_{a\sim \pi_h(\cdot|s_h)}\!\left[f_h(s_h,a)-\beta\log\frac{\pi_h(a|s_h)}{\pi_{\textnormal{ref},h}(a|s_h)}\right],
\end{align*}
with $V^\pi_{f,H+1}:=0$, together with
\begin{align*}
    \tilde{\mathcal E}_\mu(\pi,f)&:=\mathbb E_{s_1\sim\rho}\!\left[\textnormal{Var}_\mu\!\left(\sum_{h=1}^H(f_h-\mathcal T_{\beta,h}^\pi f)(s_h,a_h)\,\Big|\,s_1\right)\right],
\end{align*}
where $\mathcal T_{\beta,h}^\pi$ is the soft Bellman operator,
\begin{equation*}
    (\mathcal T_{\beta,h}^\pi f)(s_h,a_h):=r^*_h(s_h,a_h)+V^\pi_{f,h+1}\big(T_h(s_h,a_h)\big).
\end{equation*}
Following \citet{cheng2022adversarially}, $\tilde{\mathcal L}_\mu(\pi, f)$ is designed to estimate $J_\beta(\pi)-J_\beta(\mu)$, while $\tilde{\mathcal E}_\mu(\pi,f)$ performs policy evaluation through a conditional variance objective rather than a squared Bellman error; the two are equivalent in the same sense as discussed in the previous section. We can now cast the offline policy optimization problem as the following Stackelberg game \citep{von2010market}:
\begin{equation}\label{eq:ATAC_GAME}
    \begin{aligned}
        \pi^\dagger&\in \argmax_{\pi\in \Pi} \tilde{\mathcal L}_\mu(\pi, f^\pi)\quad\text{s.t.}\\
        f^\pi&\in \argmin_{f\in \mathcal F}\!\left(\tilde{\mathcal L}_\mu(\pi, f)+\eta\cdot \tilde{\mathcal E}_\mu(\pi, f)\right).
    \end{aligned}
\end{equation}
The idea is to select $\pi^\dagger$ together with an adversarial value function that is pessimistic about $\pi^\dagger$ while remaining consistent with the offline data through a low Bellman error. 
The Stackelberg formulation in \cref{eq:ATAC_GAME} is conceptually clean but bilevel, hence computationally inconvenient. A surrogate objective is a max-min game objective
\begin{equation}\label{eq:maxmin_atac_game}
    \pi^\dagger\in \argmax_{\pi\in\Pi}\min_{f\in \mathcal F}\left(\tilde{\mathcal L}_\mu(\pi, f)+\eta\,\tilde{\mathcal E}_\mu(\pi,f)\right).
\end{equation}
We now show that this step-level, MDP-style max-min game is, under a suitable condition on $\mathcal F$, equivalent to the trajectory-level, bandit-style max-min game \cref{eq:bandit_like_maxmin_game} introduced earlier, and hence to the AsymRE objective itself.
\begin{proposition}\label{prop:equi_between_zhaoran_and_maxmin}
    For each $f\in\mathcal F$ and policy $\pi$, define the induced reward model
    \begin{align*}
        r^\pi_{f,h}(s_h,a_h):=f_h(s_h,a_h)-V^\pi_{f,h+1}(T_h(s_h,a_h)),
        \qquad
        \mathcal R(\pi,\mathcal F):=\{R^\pi_f=\sum_{h}r^\pi_{f,h} : f\in \mathcal F\}.
    \end{align*}
    Suppose $\mathcal R(\pi,\mathcal F)=\mathcal R$ for every $\pi\in\Pi$, then
    \begin{equation*}
        \argmax_{\pi\in \Pi}\min_{R\in\mathcal R}\!\left(\mathcal L_\mu(\pi, R)+\eta\,\mathcal E_\mu(R)\right)
        =
        \argmax_{\pi\in\Pi}\min_{f\in \mathcal F}\!\left(\tilde{\mathcal L}_\mu(\pi, f)+\eta\,\tilde{\mathcal E}_\mu(\pi,f)\right).
    \end{equation*}
\end{proposition}
\begin{remark}
The condition $\mathcal R(\pi,\mathcal F)=\mathcal R$ for every $\pi\in\Pi$ requires the reward class induced by $\mathcal F$ not to depend on $\pi$.
 It holds when $H=1$ or when $\mathcal F$ contains all functions, but for a general $\mathcal F$ it requires additional structure and is not easy to verify. We include this proposition to connect the AsymRE objective with the game-theoretic formulation in the offline RL literature; the finite-sample analysis below does not rely on it.
\end{remark}

\paragraph{Finite Sample Analysis}
We now analyze the finite-sample guarantee of the AsymRE update. Compared with the all-policy coverage required by GRPG (\cref{thm:bound_all_policy_coverage}), the pessimism induced by the baseline shift allows us to obtain a guarantee under single-policy coverage. Define the empirical AsymRE update
\begin{equation*}
\begin{aligned}
    \nabla\mathcal E_{D,\textnormal{AsymRE}}(\pi)
    :=
    -\frac{1}{n}\sum_{i=1}^n\frac{1}{m-1}\sum_{j=1}^m
    &\left[
    R^*_{\pi,\beta}(\tau^{ij})
    -\Big(\frac{1}{m}\sum_{k=1}^m R^*_{\pi,\beta}(\tau^{ik})-\delta\Big)
    \right]\\
    &\times\left(\sum_{h=1}^H \nabla\log \pi_h(a_h^{ij}|s_h^{ij})\right).
\end{aligned}
\end{equation*}
Let
\begin{equation}\label{eq:SFT+Vairance_minimize}
    \mathcal L_{D,\textrm{SFT+VM}}(\pi):=\beta\cdot \frac{1}{nm}\sum_{i=1}^n\sum_{j=1}^m\Big[-\sum_{h=1}^H \log \pi_{h}(a_h^{ij}|s_h^{ij})\Big]+\eta\cdot\mathcal E_{D,\text{VM}}(\pi).
\end{equation}
A direct computation (\cref{eq:propto_gra_AsymRE_SFT+VM}) gives $\nabla\mathcal L_{D,\textnormal{SFT+VM}}(\pi)=2\eta\beta\nabla\mathcal E_{D,\textnormal{AsymRE}}(\pi)$ when $\delta=\frac{m-1}{2m\eta}$, so the AsymRE update is the gradient field of $\mathcal L_{D,\textnormal{SFT+VM}}$, and we analyze its minimizer. For a policy class $\Pi$ supported on $\pi_{\textnormal{ref}}$, let $N(\Pi,\epsilon)$ denote its $\epsilon$-covering number, with cover elements taken in $\Pi$, under the trajectory-level log-ratio metric
\begin{equation*}
    \textnormal{dist}(\pi,\pi'):=\sup_{s_1,\ \tau\in\textnormal{supp}(\mathbb P^{\pi_{\textnormal{ref}}}(\cdot|s_1))}\left|\log\frac{\mathbb P^\pi(\tau|s_1)}{\mathbb P^{\pi'}(\tau|s_1)}\right|.
\end{equation*}
We also need the following closure condition on $\Pi$.
\begin{assumption}\label{assum:geometric_mixture_closure}
For any $\pi_1,\pi_2\in\Pi$ and $\alpha\in[0,1]$, there exists $\pi_3\in\Pi$ such that for all $s_1$,
\begin{equation*}
    \mathbb P^{\pi_3}(\tau|s_1)\propto \mathbb P^{\pi_1}(\tau|s_1)^{\alpha}\mathbb P^{\pi_2}(\tau|s_1)^{1-\alpha}.
\end{equation*}
\end{assumption}
We then have the following result.
\begin{theorem}\label{thm:bound_single_policy_coverage}
    Let $\mathcal A$ be finite and $\beta\leq 1$. Let $\Pi$ be a policy class supported on $\pi_{\textnormal{ref}}$ satisfying \cref{assum:geometric_mixture_closure}, and let
    \begin{equation*}
        \hat \pi\in \argmin_{\pi\in\Pi}\mathcal L_{D,\textnormal{SFT+VM}}(\pi),
    \end{equation*}
    with $\eta=\sqrt{n}$. Assume \textnormal{(i)} $\pi^*\in \Pi$; \textnormal{(ii)} there exists $B<\infty$ such that $\Big|\beta\log\frac{\mathbb P^\pi(\tau|s_1)}{\mathbb P^{\pi_{\textnormal{ref}}}(\tau|s_1)}\Big|\leq B$ for all $\pi\in\Pi$, $s_1$, and $\tau\in\textnormal{supp}(\mathbb P^{\pi_{\textnormal{ref}}}(\cdot|s_1))$. Then for any comparator policy $\pi$ supported on $\pi_{\textnormal{ref}}$ with $\sup_{s_1\in \textnormal{supp}(\rho)}\beta \textnormal{KL}(\mathbb P^\pi(\cdot|s_1)\|\mathbb P^{\pi_{\textnormal{ref}}}(\cdot|s_1))\leq B$, with probability at least $1-\zeta$,
   \begin{equation*}
        J_\beta(\pi)-J_\beta(\hat \pi)
        \lesssim\left(\sum_{h=1}^H\chi^2(d_h^\pi\|d_h^\mu)
            +(H+B)^2\log\frac{N(\Pi,1/n)}{\zeta}\right)\frac{1}{\sqrt n}.
    \end{equation*}
\end{theorem}
\begin{proof}[Proof sketch]
The key step is to view $\mathcal L_{D,\textnormal{SFT+VM}}$ through the max-min game \cref{eq:bandit_like_maxmin_game}. Reparameterizing policies as reward models via $R_\pi=\beta\log\frac{\mathbb P^\pi}{\mathbb P^{\pi_{\textnormal{ref}}}}$, minimizing $\mathcal L_{D,\textnormal{SFT+VM}}$ over $\Pi$ is the same as solving the \emph{min-max} version of the game over this reward class (\cref{prop:maxmin_closed_form}): for each reward model the best-responding trajectory distribution has a closed form, and plugging it in recovers the SFT${}+{}$VM objective up to a constant. \cref{assum:geometric_mixture_closure} provides the convexity needed for a minimax theorem, which swaps min-max into max-min. In the max-min form, $\hat\pi$ maximizes the value under the most pessimistic reward model that is consistent with the data, so overestimation off the data support is controlled; this is what allows the regret to be charged only to the coverage of the comparator, giving single-policy instead of all-policy coverage (\cref{thm:guarantee_from_max_min_game}). See \cref{app:proofs_sec5}.
\end{proof}

\begin{remark}
    Note that the resulting baseline shift $\delta$ scales as $\Theta(1/\sqrt n)$. Therefore, although $\pi^*$ is no longer an exact stationary point of the AsymRE update, it remains an approximate stationary point with an $O(1/\sqrt n)$ advantage shift.
\end{remark}
\begin{remark}
    \cref{assum:geometric_mixture_closure} provides the convexity needed to invoke a minimax theorem, which is central to our analysis: it allows swapping the min-max problem solved by AsymRE into the max-min game \cref{eq:bandit_like_maxmin_game}; a similar condition appears in \citet{liu2024provably}. We discuss when it holds and what it costs in \cref{sec:discussion_assumption2}.
\end{remark}

\section{Conclusion}

We provided a \txreplace{unified analysis}{unified theoretical analysis} of off-policy policy gradient for LLM post-training, showing that at $\pi^*$ only two choices are consistent, i.e., preserve $\pi^*$ as a stationary point: trajectory-level importance weighting, which suffers from the curse of horizon, and KL regularization with an optimal soft value baseline. We \txreplace{instantiated the latter as GRPG}{identified the latter with the GRPG update, a variant of GRPO and similar empirically successful methods}, established its equivalence to conditional variance minimization and trajectory-level Bellman residual minimization, and proved a finite-sample guarantee under all-policy coverage. We further showed that a constant baseline shift implicitly induces pessimism, admits a max-min game interpretation, and upgrades the guarantee to single-policy coverage.

\section*{Acknowledgments}
We acknowledge support of the DARPA AIQ Award.

\bibliographystyle{plainnat}
\bibliography{ref}

\clearpage
\appendix
\onecolumn
\crefalias{section}{appendix}
\crefalias{subsection}{appendix}
\counterwithin{theorem}{section}
\renewcommand{\thelemma}{\thetheorem}
\renewcommand{\theproposition}{\thetheorem}
\renewcommand{\thecorollary}{\thetheorem}
\renewcommand{\theHtheorem}{\thetheorem}
\renewcommand{\theHlemma}{\thetheorem}
\renewcommand{\theHproposition}{\thetheorem}
\renewcommand{\theHcorollary}{\thetheorem}

\begin{center}
{\LARGE Appendix}
\end{center}

\section{Notation}\label{app:notation}
\cref{tab:notation} summarizes the reward, value and objective notation used throughout the paper. Symbols that appear only inside a single proof are defined where they are used.
\begin{table}[H]
\centering
\small
\renewcommand{\arraystretch}{1.25}
\begin{tabular}{p{0.26\textwidth} p{0.47\textwidth} p{0.16\textwidth}}
\toprule
Symbol & Meaning & Defined in \\
\midrule
\multicolumn{3}{l}{\emph{Rewards and returns}}\\
$r^*_h$, $R^*(\tau)$ & true per-step reward and trajectory reward  & Sec.~2.1\\
$r_h$, $R$, $\mathcal R$ & additive reward model and reward class & Sec.~5\\
$R^*_{\pi,\beta}(\tau)$ & KL-regularized return $R^*(\tau)-\beta\sum_h\log\frac{\pi_h(a_h|s_h)}{\pi_{\textnormal{ref},h}(a_h|s_h)}$ & Sec.~2.1\\
$R_\pi$, $\mathcal R_\Pi$ & policy-induced reward $\beta\log\frac{\mathbb P^\pi(\tau|s_1)}{\mathbb P^{\pi_{\textnormal{ref}}}(\tau|s_1)}$ and $\{R_\pi:\pi\in\Pi\}$ & App.~\labelcref{sec:proof_thm_single_policy}\\
\midrule
\multicolumn{3}{l}{\emph{Value functions and induced policies}}\\
$Q^\pi_{\beta,h}$, $V^\pi_{\beta,h}$ & soft $Q$ and $V$ of $\pi$ under $r^*$ & Sec.~2.1\\
$Q^*_{\beta,h}$, $V^*_{\beta,h}$ & optimal soft $Q$ and $V$ under $r^*$ & Sec.~2.1\\
$V^*_1(s_1)$ & unregularized optimal value $\max_\tau R^*(\tau)$ & Sec.~2.1\\
$f=(f_h)_{h}$, $\mathcal F$ & candidate $Q$-function and value function class & Sec.~4.1\\
$\pi_f$, $Z_{f,h}$, $V_{f,h}$ & soft-greedy policy $\pi_{f,h}\propto\pi_{\textnormal{ref},h}\exp(f_h/\beta)$, its normalizer, and $V_{f,h}=\beta\log Z_{f,h}$ & Sec.~4.1, App.~\labelcref{sec:revisit_KL_RL}\\
$V^\pi_{f,h}$ & soft value of $f$ under $\pi$, $\mathbb E_{a\sim\pi_h}[f_h-\beta\log\frac{\pi_h}{\pi_{\textnormal{ref},h}}]$ & Sec.~5\\
$\mathcal T^*_{\beta,h}$, $\mathcal T^\pi_{\beta,h}$ & soft Bellman optimality operator and soft Bellman operator  & Sec.~4.1, Sec.~5\\
$Q_{R,h}$, $V_{R,h}$, $\pi_R$, $\Pi_{\mathcal R}$ & soft dynamic programming under reward model $R$, the induced policy, and $\{\pi_R:R\in\mathcal R\}$ & App.~\labelcref{sec:maxmin_guarantee}\\
$Z_R(s_1)$ & normalizer with $\mathbb P^{\pi_R}(\tau|s_1)=\mathbb P^{\pi_{\textnormal{ref}}}(\tau|s_1)e^{R(\tau)/\beta}/Z_R(s_1)$ & App.~\labelcref{sec:maxmin_guarantee}\\
\midrule
\multicolumn{3}{l}{\emph{Objectives}}\\
$\nabla\mathcal E_{\mu,\textnormal{PG}}$; $\omega_\pi$, $A_\pi$, $b_\pi$ & off-policy policy-gradient update template; weight, advantage and baseline & Eq.~\eqref{eq:off_policy_PG_template}\\
$\nabla\mathcal E_{\mu,\textnormal{GRPG}}$, $\nabla\mathcal E_{D,\textnormal{GRPG}}$ & GRPG update, population and empirical & Sec.~4.1, 4.2\\
$\mathcal E_{\mu,\textnormal{VM}}$, $\mathcal E_{D,\textnormal{VM}}$ & conditional variance of $R^*_{\pi,\beta}$ given $s_1$, population and empirical & Sec.~4.1, 4.2\\
$\mathcal E_{\mu,\textnormal{TBRM}}$ & trajectory Bellman residual & Sec.~4.1\\
$\nabla\mathcal E_{\mu,\textnormal{AsymRE}}$, $\nabla\mathcal E_{D,\textnormal{AsymRE}}$ & AsymRE update, population and empirical & Sec.~5\\
$\mathcal L_{D,\textnormal{SFT+VM}}$ & $\beta\cdot\textnormal{SFT}+\eta\cdot\mathcal E_{D,\textnormal{VM}}$ & Sec.~5\\
$\mathcal L_\mu(\pi,R)$, $\mathcal E_\mu(R)$ & trajectory-level game: value gap under reward model $R$; variance of $R^*-R$ & Sec.~5\\
$\tilde{\mathcal L}_\mu(\pi,f)$, $\tilde{\mathcal E}_\mu(\pi,f)$ & step-level counterparts with a value function $f$ & Sec.~5\\
$\mathcal P$; $\phi$, $\phi_D$; $\mathcal L_D$, $\mathcal E_D$ & trajectory distributions supported on $\pi_{\textnormal{ref}}$; game payoff $\mathcal L+\eta\mathcal E$ and its empirical version & App.~\labelcref{sec:maxmin_guarantee}\\
$\mathcal L^\circ_\mu$, $\mathcal L^\circ_D$, $\phi^\circ_D$ & the same with the $\beta\textnormal{KL}(\mathbb P^\mu\|\mathbb P^{\pi_{\textnormal{ref}}})$ term removed & App.~\labelcref{sec:maxmin_guarantee}\\
\bottomrule
\end{tabular}
\caption{Summary of notation.}
\label{tab:notation}
\end{table}

\clearpage
\cref{tab:abbreviations} collects the abbreviations used for the objectives and algorithms discussed in the paper.
\begin{table}[H]
\centering
\small
\renewcommand{\arraystretch}{1.25}
\begin{tabular}{p{0.13\textwidth} p{0.53\textwidth} p{0.24\textwidth}}
\toprule
Abbreviation & Meaning & Reference \\
\midrule
IS & importance sampling & Sec.~3\\
PG & policy gradient & Sec.~2.2\\
GRPG & group-relative policy gradient; the KL-route update studied here, a variant of the GRPO update & Sec.~4, \citet{shao2024deepseekmath}\\
VM & conditional variance minimization; the objective whose gradient field is the GRPG update & Sec.~4.1, \citet{tang2025rl}\\
TBRM & trajectory Bellman residual minimization; the value-based objective whose minimizers induce the same policies as VM & Sec.~4.1, \citet{yuan2025trajectory}\\
AsymRE & asymmetric REINFORCE; GRPG with a constant shift in the grouped-mean baseline & Sec.~5, \citet{arnal2025asymmetric}\\
ATAC & adversarially trained actor-critic; the step-level game-theoretic formulation of offline RL & \citet{cheng2022adversarially}\\
\midrule
GRPO & group-relative policy optimization & \citet{shao2024deepseekmath}\\
SFT & supervised fine-tuning; here, toward the behavior policy $\mu$ & Sec.~5\\
BC & behavior cloning & App.~\labelcref{app:synthetic}\\
\bottomrule
\end{tabular}
\caption{Abbreviations for the objectives and algorithms discussed in the paper.}
\label{tab:abbreviations}
\end{table}

\section{Related Work}\label{app:related_work}

\paragraph{Classical Off-Policy Policy Gradient}
Off-policy policy gradient has a long history, and most methods make the update valid by recovering a target-policy quantity from behavior data. Per-decision importance sampling \citep{precup2000eligibility}, originally developed for evaluation, and its use in policy optimization \citep{jiang2016doubly,wang2017sample,espeholt2018impala} recover expectations under $\pi$; density-ratio methods recover the state distribution $d^\pi$ \citep{liu2019off}; off-policy actor-critic \citep{degris2012off} recovers $Q^\pi$ through a learned critic, with per-step ratios on the actor side, and its update approximates the gradient of a behavior-weighted objective rather than $J(\pi)$. In all three cases the target is the on-policy gradient or a surrogate of it, and validity follows from unbiasedness. The KL route studied in this paper is different in kind: it recovers none of these quantities and is not unbiased; it is consistent because, under KL regularization with a grouped-mean baseline, the optimality condition holds along every trajectory and the advantage vanishes at $\pi^*$ (\cref{sec:kl-route}). This is what allows it to dispense with importance weights, density ratios, and critics, at the price of requiring deterministic transitions and a coverage condition. Relatedly, in the on-policy setting the baseline only affects variance, and \citet{peters2008reinforcement} derive the baseline that minimizes it; in our off-policy setting the baseline determines whether $\pi^*$ remains a stationary point at all, and a constant shift in it implements pessimism (\cref{sec:baseline-const-shift}).

\paragraph{Policy Gradient and LLM Post-Training}
KL-regularized policy gradient approach is the de facto template for LLM post-training and reasoning fine-tuning, instantiated by PPO~\citep{schulman2017proximal}, group-relative variants such as GRPO and DAPO~\citep{shao2024deepseekmath, yu2025dapo}, and REINFORCE-leave-one-out style updates~\citep{williams1992simple, ahmadian2024back}. A parallel line of recent work seeks to make these updates principled when the rollout policy differs from the learner: AGRO~\citep{tang2025rl} casts policy improvement as conditional variance minimization in the contextual-bandit setting; TBRM~\citep{yuan2025trajectory} applies trajectory-level Bellman residual minimization with finite-sample guarantees; and AsymRE~\citep{arnal2025asymmetric} introduces a constant baseline shift in REINFORCE that asymmetrically up- or down-weights positive and negative samples. Closely related, \citet{yao2025group} provide an empirical study arguing that group-relative REINFORCE is implicitly off-policy---an observation our theory makes precise. We unify these threads through (i) a stationarity-based necessary characterization showing that, within a broad family of off-policy PG objectives, at the optimal policy, only trajectory-level importance weighting or KL regularization paired with an optimal soft value baseline are consistent, i.e., preserve the optimal policy as a stationary point of the off-policy update (\cref{prop:neccesary_condition_correcting_oppg}); (ii) a value-based equivalence between GRPG and trajectory Bellman residual minimization that ports TBRM's finite-sample guarantees to a policy-gradient method (\cref{prop:grpg_tbrm_equivalence,thm:bound_all_policy_coverage}); and (iii) an upgrade of these guarantees to single-policy coverage when AsymRE's baseline shift is interpreted as pessimism (\cref{thm:bound_single_policy_coverage}).

\paragraph{Value-Policy Equivalence in Reinforcement Learning}
Our proportional-gradient identity (\cref{eq:PG_TO_VM}) and GRPG--TBRM equivalence (\cref{prop:grpg_tbrm_equivalence}) build on a long line of work in KL-regularized (equivalently, entropy-regularized) reinforcement learning, originating from the maximum-entropy formulation of \citet{ziebart2008maximum,ziebart2010modeling,neu2017unified,geist2019theory}, in which a Kullback--Leibler penalty toward a reference distribution is added to the reward objective. This framework now serves as the standard backbone for LLM post-training and admits both policy-based and value-based instantiations, whose connections have been the subject of substantial study. One line of algorithms, in the soft Q-learning style---SQL~\citep{haarnoja2017reinforcement}, the policy-gradient/soft-Q-learning equivalence of \citet{schulman2017equivalence}, and its adaptation to text generation~\citep{guo2022efficient}---formulates KL-regularized policy improvement as fitting a soft Q-function via a single-step Bellman residual. A parallel actor-critic line, including Path Consistency Learning (PCL) of \citet{nachum2017bridging} and Soft Actor-Critic~\citep{haarnoja2018soft}, exploits the optimality conditions of entropy-regularized policies as a multi-step softmax temporal-consistency relation along arbitrary action sequences, with recent LLM-side instantiations such as \citet{ji2024enhancing,wang2025offline}. TBRM~\citep{yuan2025trajectory} brings this value-based line to LLM reasoning via a clean trajectory-level Bellman residual minimization objective with finite-sample guarantees; see also \citet{wang2026off} for related work in this direction. Our analysis extends these connections to the trajectory level under off-policy data: the proportional-gradient identity recasts the GRPG (policy-gradient) update as the gradient of the VM objective, whose minimizers coincide with those of TBRM (\cref{prop:grpg_tbrm_equivalence}), thereby porting TBRM's finite-sample function-approximation regret guarantees to the policy-gradient side (\cref{thm:bound_all_policy_coverage}).

\paragraph{Offline Reinforcement Learning}
Our analysis of AsymRE (\cref{sec:baseline-const-shift}) plugs into a rich body of theoretical results in offline RL; see \citet{jiang2025offline} for a recent survey, and \citet{levine2020offline, fujimoto2019off, lange2012batch} for earlier perspectives on the offline (or batch) regime. A central principle is \emph{pessimism in the face of uncertainty}, which counteracts distribution shift by penalizing value estimates in regions poorly covered by data \citep{jin2021pessimism, xie2021bellman, cheng2022adversarially}. Closest to our algorithmic perspective are \emph{pessimistic and adversarially trained actor-critic methods}~\citep{xie2021bellman, cheng2022adversarially, liu2024provably}, where the actor maximizes a value function chosen pessimistically from a confidence set; the Stackelberg formulation we use to interpret AsymRE as a max-min game (\cref{eq:ATAC_GAME}) traces to the ATAC framework of \citet{cheng2022adversarially}, originally developed for the infinite-horizon, no-KL setting, and was extended to the contextual-bandit, KL-regularized setting under the name RPO by \citet{liu2024provably}. Compared to RPO, which operates in a preference-learning setting and enforces value consistency through a DPO-style loss, our analysis enforces consistency through a conditional-variance loss and operates at the trajectory level. A complementary line of work imposes pessimism implicitly through behavior regularization, most prominently TD3+BC~\citep{fujimoto2021minimalist}, where a behavior-cloning term toward the offline distribution is added to the actor loss; the AsymRE gradient decomposes into a variance-minimizing critic plus exactly such a behavior-cloning term toward the offline policy, which we view as a natural LLM-side instantiation of the TD3+BC principle.

\section{Technical Lemmas}

First, we introduce this critical change-of-trajectory lemma.
\begin{lemma}\label{lemma:change_of_trajectory}
\textnormal{(Change of Measure, Lemma 1 of \cite{yuan2025trajectory})}
    Given an MDP $M=(H, \mathcal S, \mathcal A, T, r^*, \rho)$ and two Markovian policies $\pi,\mu$, it holds, for any $g=(g_h)_{h=1}^H$ with $g_h:\mathcal S\times\mathcal A\rightarrow\mathbb R$,
    \begin{equation*}
        {\left(\mathbb E_\pi\left[\sum_{h=1}^Hg_h(s_h,a_h)\right]\right)^2}\leq \left(1+\sum_{h=1}^H \chi^2(d_h^\pi\|d_h^\mu)\right)\cdot \mathbb E_\mu\left[\left(\sum_{h=1}^H g_h(s_h,a_h)\right)^2\right].
    \end{equation*}
\end{lemma}

\section{Proofs from Section 3}

\subsection{Revisitation of KL-Regularized RL}\label{sec:revisit_KL_RL}

For a given policy $\pi$, define $V_{\beta,H+1}^{\pi}(s):=0$. For $h=H,H-1,\ldots,1$, recursively define
\begin{equation*}
       Q^\pi_{\beta,h}(s,a):=
    r_h^*(s,a)+V_{\beta,h+1}^{\pi}(T_h(s,a)),
\end{equation*}
and
\begin{equation*}
     V^{\pi}_{\beta,h}(s):=
    \mathbb E_{a\sim \pi_h(\cdot|s)}
    \left[Q^\pi_{\beta,h}(s,a)-
        \beta\log
        \frac{\pi_h(a|s)}
        {\pi_{\textnormal{ref},h}(a|s)}
    \right].
\end{equation*}
Now we define the Bellman operator $\mathcal T_{\beta}^\pi$. For any action-value function $f=(f_h)_{h=1}^H$, define
\begin{equation*}
    V^\pi_{f,h}(s)
    :=
    \mathbb E_{a\sim \pi_h(\cdot|s)}
    \left[
        f_h(s,a)
        -
        \beta\log
        \frac{\pi_h(a|s)}
        {\pi_{\textnormal{ref},h}(a|s)}
    \right],
\end{equation*}
and
\begin{equation*}
    (\mathcal T_{\beta,h}^\pi f)(s,a)
    :=
    r_h^*(s,a)+V^\pi_{f,h+1}(T_h(s,a)).
\end{equation*}
The soft-greedy policy induced by $f$ is
\begin{equation*}
    \pi_{f,h}(\cdot|s)
    :=
    \argmax_{\pi_h(\cdot|s)\in \Delta(\mathcal A)}
    \mathbb E_{a\sim \pi_h(\cdot|s)}
    \left[
        f_h(s,a)
        -
        \beta\log
        \frac{\pi_h(a|s)}
        {\pi_{\textnormal{ref},h}(a|s)}
    \right].
\end{equation*}
It has the closed form
\begin{equation}\label{eq:f_induced_policy}
    \pi_{f,h}(a|s)
    =
    \frac{
        \pi_{\textnormal{ref},h}(a|s)
        \exp(f_h(s,a)/\beta)}{Z_{f,h}(s)},
\end{equation}
where
\begin{equation*}
    Z_{f,h}(s):=\int_{\mathcal A}
    \pi_{\textnormal{ref},h}(a|s)
    \exp\left(\frac{f_h(s,a)}{\beta}\right)da.
\end{equation*}
The corresponding maximum value (through computation) is,
\begin{equation*}
    V_{f,h}(s)
    =
    \beta\log Z_{f,h}(s).
\end{equation*}
Therefore, under deterministic transition $s_{h+1}=T_h(s_h,a_h)$, the soft Bellman optimality operator is
\begin{align*}
    (\mathcal T^*_{\beta,h} f)(s_h,a_h)
    :=&r_h^*(s_h,a_h)+\max_{\pi_{h+1}(\cdot|s_{h+1})\in \Delta(\mathcal A)}\mathbb E_{a\sim \pi_{h+1}(\cdot|s_{h+1})}\bigg[f_{h+1}(s_{h+1},a)\\
    &\qquad-\beta\log\frac{\pi_{h+1}(a|s_{h+1})}{\pi_{\textnormal{ref},h+1}(a|s_{h+1})}\bigg]\\
    =&r_h^*(s_h,a_h)
    +
    V_{f,h+1}(s_{h+1})\\
    =&r_h^*(s_h,a_h)
    +
    \beta
    \log
    \int_{\mathcal A}
    \pi_{\textnormal{ref},h+1}(a|s_{h+1})
    \exp\left(
        \frac{f_{h+1}(s_{h+1},a)}{\beta}
    \right)da,
\end{align*}
where $s_{h+1}=T_h(s_h,a_h)$ and $V_{f,H+1}(s):=0$.

The optimal soft Q-function and value function are defined by
\begin{equation*}
    Q^*_{\beta,h}(s,a)
    :=
    r_h^*(s,a)
    +
    V^*_{\beta,h+1}(T_h(s,a)),
    \qquad
    V^*_{\beta,H+1}(s):=0,
\end{equation*}
with
\begin{equation*}
    V^*_{\beta,h}(s)
    =
    \beta
    \log
    \int_{\mathcal A}
    \pi_{\textnormal{ref},h}(a|s)
    \exp\left(
        \frac{Q^*_{\beta,h}(s,a)}{\beta}
    \right)da.
\end{equation*}
The optimal policy satisfies
\begin{equation}\label{eq:step_optimality_condition_kl}
    \begin{aligned}
        \pi_h^*(a|s)&=
    \frac{
        \pi_{\textnormal{ref},h}(a|s)
        \exp(Q^*_{\beta,h}(s,a)/\beta)
    }{
        \int_{\mathcal A}
        \pi_{\textnormal{ref},h}(a'|s)
        \exp(Q^*_{\beta,h}(s,a')/\beta)da'}\\
    \Rightarrow \beta\log\frac{\pi_h^*(a|s)}{\pi_{\textnormal{ref},h}(a|s)}&=Q_{\beta,h}^*(s,a)-V_{\beta,h}^*(s)\\
    &=r_h^*(s,a)+V^*_{\beta,h+1}(s_{h+1})-V_{\beta,h}^*(s)
    \end{aligned}
\end{equation}
Therefore, by telescoping,
\begin{equation}\label{eq:trajectory_reward_KL_optimality}
    \sum_{h=1}^H
    \left[
        r_h^*(s_h,a_h)
        -
        \beta
        \log
        \frac{\pi_h^*(a_h|s_h)}
        {\pi_{\textnormal{ref},h}(a_h|s_h)}
    \right]
    =
    V_{\beta,1}^*(s_1),
\end{equation}
where \(V_{\beta,H+1}^*(s_{H+1})=0\).

\subsection{Proof of \texorpdfstring{\cref{prop:neccesary_condition_correcting_oppg}}{Proposition 1}}\label{proof_of_two_paths_prop}
\begin{proof}
Fix $s_1\in\textnormal{supp}(\rho)$, $h\in[H]$, $a\in\mathcal A$, and a state $s=(s_1,a_1,\ldots,a_{h-1})$ with initial state $s_1$. Since $\mathbb E_{\tau\sim\mu}$ averages over the initial state, \cref{eq:stationarity_condition} for this $(h,s,a)$ reads
\begin{equation*}
    \sum_{s_1'}\rho(s_1')\,\mathbb E_{\tau\sim\mathbb P^\mu(\cdot|s_1')}\Big[\omega_\pi(\tau)A_\pi(\tau)\Big(\mathbf 1\{s_h(\tau)=s,\ a_h(\tau)=a\}-\mathbf 1\{s_h(\tau)=s\}\pi_h(a|s)\Big)\Big]=0.
\end{equation*}
Both indicators contain $\mathbf 1\{s_h(\tau)=s\}$, and in the autoregressive setting a trajectory starting from $s_1'\neq s_1$ has $s_h(\tau)=(s_1',a_1,\ldots,a_{h-1})\neq s$. Hence every summand with $s_1'\neq s_1$ is zero, and dividing the remaining one by $\rho(s_1)>0$ gives
\begin{equation}\label{eq:stationarity_condition_conditional}
    \mathbb E_{\tau\sim\mathbb P^\mu(\cdot|s_1)}\Big[\omega_\pi(\tau)A_\pi(\tau)\Big(\mathbf 1\{s_h(\tau)=s,\ a_h(\tau)=a\}-\mathbf 1\{s_h(\tau)=s\}\pi_h(a|s)\Big)\Big]=0.
\end{equation}
For a fixed $s_1$, define $M(\tau):=\mathbb P^\mu(\tau|s_1)\omega_{\pi^*}(\tau)A_{\pi^*}(\tau)$, and for any prefix $s_h=(s_1,a_1,\ldots,a_{h-1})$, define $m(s_h):=\sum_{\tau:\ \textnormal{has prefix }s_h}M(\tau)$. In the autoregressive setting, $m(s_{H+1})=M(\tau)$ when $s_{H+1}(\tau)=s_{H+1}$. Condition \cref{eq:stationarity_condition_conditional} states exactly that, for every $h$, $s_h$ and $a$,
\begin{equation*}
    \sum_{\tau:\ s_h(\tau)=s_h,\ a_h(\tau)=a}M(\tau)=\pi^*_h(a|s_h)\sum_{\tau:\ s_h(\tau)=s_h}M(\tau).
\end{equation*}
Since $s_{h+1}=(s_h,a)$, therefore $m(s_{h+1})=\pi^*_h(a|s_h)m(s_h)$. Iterating, for every trajectory $\tau$ from $s_1$,
\begin{equation}\label{eq:master_identity}
    M(\tau)=m(s_{H+1})=m(s_1)\prod_{h=1}^H\pi^*_h(a_h|s_h)=m(s_1)\mathbb P^{\pi^*}(\tau|s_1),
\end{equation}
where for $\tau\notin\textnormal{supp}(\mathbb P^\mu(\cdot|s_1))$ the left side is $0$ by definition of $M$.

Now we consider two cases. First consider the case without KL regularization, where $A_{\pi^*}(\tau)=R^*(\tau)-b_{\pi^*}(s_1)$. Fix $\tau\in\textnormal{supp}(\mathbb P^\mu(\cdot|s_1))$ with $R^*(\tau)\neq b_{\pi^*}(s_1)$, so that $A_{\pi^*}(\tau)\neq 0$ and \cref{eq:master_identity} gives
\begin{equation*}
    \omega_{\pi^*}(\tau)=\frac{m(s_1)}{A_{\pi^*}(\tau)}\cdot\frac{\mathbb P^{\pi^*}(\tau|s_1)}{\mathbb P^\mu(\tau|s_1)}.
\end{equation*}
Without KL regularization, $\pi^*$ is supported on the set of reward-maximizing trajectories $\mathcal T_{\textnormal{opt}}(s_1):=\{\tau: R^*(\tau)=V^*_1(s_1)\}$. If $\tau\notin \mathcal T_{\textnormal{opt}}(s_1)$, then $\mathbb P^{\pi^*}(\tau|s_1)=0$ and hence $\omega_{\pi^*}(\tau)=0$. If $\tau\in \mathcal T_{\textnormal{opt}}(s_1)$, then $R^*(\tau)=V^*_1(s_1)\neq b_{\pi^*}(s_1)$ and the prefactor equals $m(s_1)/(V^*_1(s_1)-b_{\pi^*}(s_1))$. Both cases take the form
\begin{equation*}
    \omega_{\pi^*}(\tau)=K(s_1)\frac{\mathbb P^{\pi^*}(\tau|s_1)}{\mathbb P^\mu(\tau|s_1)}=K(s_1)\prod_{h=1}^H\frac{\pi^*_h(a_h|s_h)}{\mu_h(a_h|s_h)},
\end{equation*}
with $K(s_1):=m(s_1)/(V^*_1(s_1)-b_{\pi^*}(s_1))$ if $V^*_1(s_1)\neq b_{\pi^*}(s_1)$, and $K(s_1):=0$ otherwise, in which case every $\tau$ with $R^*(\tau)\neq b_{\pi^*}(s_1)$ lies outside $\mathcal T_{\textnormal{opt}}(s_1)$.

Now consider the case with KL regularization. By \cref{eq:trajectory_reward_KL_optimality}, $R^*_{\pi^*,\beta}(\tau)=V^*_{\beta,1}(s_1)$ for every $\tau\in\textnormal{supp}(\mathbb P^{\pi_{\textnormal{ref}}}(\cdot|s_1))\supseteq\textnormal{supp}(\mathbb P^\mu(\cdot|s_1))$, hence $A_{\pi^*}(\tau)=V^*_{\beta,1}(s_1)-b_{\pi^*}(s_1)$ is independent of $\tau$. If $V^*_{\beta,1}(s_1)\neq b_{\pi^*}(s_1)$, dividing \cref{eq:master_identity} by this constant gives, for every $\tau\in\textnormal{supp}(\mathbb P^\mu(\cdot|s_1))$,
\begin{equation*}
    \omega_{\pi^*}(\tau)=K(s_1)\frac{\mathbb P^{\pi^*}(\tau|s_1)}{\mathbb P^\mu(\tau|s_1)}=K(s_1)\prod_{h=1}^H\frac{\pi^*_h(a_h|s_h)}{\mu_h(a_h|s_h)},\qquad K(s_1):=\frac{m(s_1)}{V^*_{\beta,1}(s_1)-b_{\pi^*}(s_1)}.
\end{equation*}
Otherwise, $b_{\pi^*}(s_1)=V^*_{\beta,1}(s_1)=\mathbb E_\mu\left[R^*_{\pi^*,\beta}(\tau)|s_1\right]$, where the last equality uses $\textnormal{supp}(\mathbb P^\mu(\cdot|s_1))\subseteq\textnormal{supp}(\mathbb P^{\pi_{\textnormal{ref}}}(\cdot|s_1))$. This proves the claim.
\end{proof}

\section{Proofs from Section 4}

\subsection{Proof of \texorpdfstring{\cref{prop:grpg_tbrm_equivalence}}{Proposition 2}}
\begin{proof}
    Notice that by \cref{eq:f_induced_policy}, it holds,
\begin{equation}\label{eq:pi_f_optimality_condition}
    f_h(s,a)=\beta \log\frac{\pi_{f,h}(a|s)}{\pi_{\textrm{ref},h}(a|s)}+V_{f,h}(s).
\end{equation}
Then we calculate the trajectory-level Bellman residual,
\begin{equation*}
    \begin{aligned}
        \sum_{h=1}^H \left(f_h(s_h,a_h)-(\mathcal T^*_{\beta,h} f)(s_h,a_h)\right)&=\sum_{h=1}^H \left(f_h(s_h,a_h)-r_h^*(s_h,a_h)-V_{f,h+1}(s_{h+1})\right)\\
        &=\sum_{h=1}^H\bigg(\beta\log\frac{\pi_{f,h}(a_h|s_h)}{\pi_{\textrm{ref},h}(a_h|s_h)}+V_{f,h}(s_h)\\
        &\qquad\qquad-r_h^*(s_h,a_h)-V_{f,h+1}(s_{h+1})\bigg)\\
        &= -R^*_{\pi_f,\beta}(\tau)+V_{f,1}(s_1).
    \end{aligned}
\end{equation*}
Let $m_{\pi_f}(s_1)=\mathbb E_\mu[R^*_{\pi_f,\beta}(\tau)|s_1]$, then notice that it holds,
\begin{equation*}
    \begin{aligned}
        \mathcal E_{\mu,\textrm{TBRM}}(f)&= \mathbb E_{\mu}\left[\left(-R^*_{\pi_f,\beta}(\tau)+V_{f,1}(s_1)\right)^2\right]\\
        &=\mathbb E_{s_1\sim\rho}\left[\textrm{Var}\left(R^*_{\pi_f,\beta}(\tau)\big|s_1\right)\right]+\mathbb E\left[\left(m_{\pi_f}(s_1)-V_{f,1}(s_1)\right)^2\right]\\
        &=\mathcal E_{\mu,\textnormal{VM}}(\pi_f)+\mathbb E\left[\left(m_{\pi_f}(s_1)-V_{f,1}(s_1)\right)^2\right]\geq \mathcal E_{\mu,\textnormal{VM}}(\pi_f).
    \end{aligned}
\end{equation*}
On the other hand, for any $f\in \mathcal F$, let $c_f(s_1)=m_{\pi_f}(s_1)-V_{f,1}(s_1)$, and define,
\begin{equation*}
    f^{c_f}_1:=f_1+c_f, f_{h}^{c_f}:= f_h, h\geq 2.
\end{equation*}
This implies that $\pi_{f^{c_f}}=\pi_f$ by \cref{eq:f_induced_policy}. And by \cref{eq:pi_f_optimality_condition},
\begin{equation*}
    V_{f^{c_f},1}(s_1)=V_{f,1}(s_1)+c_f(s_1)=m_{\pi_f}(s_1), 
\end{equation*}
plug this in we have,
\begin{equation}\label{eq:equation_demean_tbrm_vm}
    \begin{aligned}
        \mathcal E_{\mu,\textnormal{TBRM}}(f^{c_f})=\mathcal E_{\mu,\textnormal{VM}}(\pi_f),
    \end{aligned}
\end{equation}
this implies,
\begin{equation*}
    \min_{\pi\in \Pi_{\mathcal F}}\mathcal E_{\mu,\textrm{VM}}(\pi)=\min_{f\in \mathcal F}\mathcal E_{\mu,\textnormal{VM}}(\pi_f)\geq \min_{f\in\mathcal F}\mathcal E_{\mu,\textnormal{TBRM}}(f^{c_f})\geq \min_{f\in\mathcal F}\mathcal E_{\mu,\textnormal{TBRM}}(f).
\end{equation*}
Above, the first equation is by definition of $\Pi_{\mathcal F}$; the first inequality is by \cref{eq:equation_demean_tbrm_vm}.
Combined with $\mathcal E_{\mu,\textrm{TBRM}}(f)\geq \mathcal E_{\mu,\textnormal{VM}}(\pi_f)$, this shows,
\begin{equation*}
    \min_{\pi\in \Pi_{\mathcal F}}\mathcal E_{\mu,\textrm{VM}}(\pi)=\min_{f\in\mathcal F}\mathcal E_{\mu,\textnormal{TBRM}}(f).
\end{equation*}
Now we prove the set equality. If $f^*\in \argmin_f \mathcal E_{\mu,\textrm{TBRM}}(f)$, then,
\begin{equation*}
    \min_{\pi\in\Pi_{\mathcal F}}\mathcal E_{\mu,\textrm{VM}}(\pi)=\mathcal E_{\mu,\textrm{TBRM}}(f^*)\geq \mathcal E_{\mu,\textrm{VM}}(\pi_{f^*})\geq \min_{\pi\in\Pi_{\mathcal F}}\mathcal E_{\mu,\textrm{VM}}(\pi),
\end{equation*}
this implies $\pi_{f^*}\in \argmin_{\pi\in\Pi_{\mathcal F}}\mathcal E_{\mu,\textrm{VM}}(\pi)$.
On the other hand, if $\pi'\in \argmin_{\pi\in\Pi_{\mathcal F}}\mathcal E_{\mu,\textrm{VM}}(\pi)$, and suppose $\pi'=\pi_{f}$, we construct $f^{c_f}$ as above, then $\pi_{f^{c_f}}=\pi_f=\pi'$, then by \cref{eq:equation_demean_tbrm_vm}
\begin{equation*}
    \mathcal E_{\mu,\textrm{TBRM}}(f^{c_f})=\mathcal E_{\mu,\textrm{VM}}(\pi_f)=\min_{\pi\in \Pi_{\mathcal F}}\mathcal E_{\mu,\textrm{VM}}(\pi)=\min_{f\in\mathcal F}\mathcal E_{\mu,\textnormal{TBRM}}(f),
\end{equation*}
this implies $f^{c_f}\in \argmin_{f\in\mathcal F}\mathcal E_{\mu,\textrm{TBRM}}(f)$ and $\pi_{f^{c_f}}=\pi'$. Therefore we have the set equality,
\begin{equation*}
    \argmin_{\pi\in \Pi_{\mathcal F}}\mathcal E_{\mu,\textrm{VM}}(\pi)=\left\{\pi_{f^*}, f^*\in \argmin_{f\in\mathcal F}\mathcal E_{\mu,\textrm{TBRM}}(f)\right\}.
\end{equation*}
This concludes the proof.
\end{proof}

\subsection{Proof of \texorpdfstring{\cref{thm:bound_all_policy_coverage}}{Theorem 3}}\label{sec:proof_thm_bound_all_policy}
First we state and prove a few technical lemma we need to establish \cref{thm:bound_all_policy_coverage}.

\begin{lemma}\label{lemma:reward_to_KL}
    \begin{equation*}
        J_\beta(\pi^*)-J_\beta(\pi)=\beta \mathbb E_{\pi}\left[\sum_{h=1}^H \textnormal{KL}(\pi_h(\cdot|s_h)\| \pi_h^*(\cdot|s_h))\right]
    \end{equation*}
\end{lemma}
\begin{proof}
From the soft optimality condition derived in \cref{eq:step_optimality_condition_kl}, for any trajectory,
\begin{equation*}
    r_h^*(s_h,a_h)
    -
    \beta
    \log
    \frac{\pi_h^*(a_h|s_h)}
    {\pi_{\textnormal{ref},h}(a_h|s_h)}
    =
    V_{\beta,h}^*(s_h)-V_{\beta,h+1}^*(s_{h+1}).
\end{equation*}
Hence,
    \begin{align*}
    r_h^*(s_h,a_h)
    -
    \beta
    \log
    \frac{\pi_h(a_h|s_h)}
    {\pi_{\textnormal{ref},h}(a_h|s_h)}
    &=
    V_{\beta,h}^*(s_h)-V_{\beta,h+1}^*(s_{h+1})
    -
    \beta
    \log
    \frac{\pi_h(a_h|s_h)}
    {\pi_h^*(a_h|s_h)}.
\end{align*}

Summing over $h$ and using $V_{\beta,H+1}^*=0$ gives
\begin{equation}\label{eq:trajectory_reward_logratio_identity}
    R^*_{\pi,\beta}(\tau)
    =
    \sum_{h=1}^H
    \left[
        r_h^*(s_h,a_h)
        -
        \beta
        \log
        \frac{\pi_h(a_h|s_h)}
        {\pi_{\textnormal{ref},h}(a_h|s_h)}
    \right]
    =
    V_{\beta,1}^*(s_1)
    -
    \beta
    \sum_{h=1}^H
    \log
    \frac{\pi_h(a_h|s_h)}
    {\pi_h^*(a_h|s_h)}.
\end{equation}
Taking expectation under $\tau\sim\mathbb P^\pi(\cdot|s_1)$ and $s_1\sim\rho$, we obtain
\[
    J_\beta(\pi)
    =
    J_\beta(\pi^*)
    -
    \beta
    \mathbb E_{\pi}
    \left[
        \sum_{h=1}^H
        \log
        \frac{\pi_h(a_h|s_h)}
        {\pi_h^*(a_h|s_h)}
    \right].
\]
Finally, conditioning on $s_h$ yields
\[
    \mathbb E_{a_h\sim\pi_h(\cdot|s_h)}
    \left[
        \log
        \frac{\pi_h(a_h|s_h)}
        {\pi_h^*(a_h|s_h)}
    \right]
    =
    \textnormal{KL}
    \left(
        \pi_h(\cdot|s_h)
        \|
        \pi_h^*(\cdot|s_h)
    \right).
\]
Therefore,
\[
    J_\beta(\pi^*)-J_\beta(\pi)
    =
    \beta
    \mathbb E_{\pi}
    \left[
        \sum_{h=1}^H
        \textnormal{KL}
        \left(
            \pi_h(\cdot|s_h)
            \|
            \pi_h^*(\cdot|s_h)
        \right)
    \right].
\]
\end{proof}

\begin{lemma}\label{lemma:KL_to_variance}
    Under \cref{assume:all_policy_state_action_coverage} and $\pi^*\in\Pi$, for any $\pi\in\Pi$, it holds,
    \begin{equation*}
        \mathbb E_{s_1\sim\rho}\left[\textnormal{KL}\left(\mathbb P^\pi(\cdot|s_1)\|\mathbb P^{\pi^*}(\cdot|s_1)\right)\right]\leq 2\sqrt{(1+C_{\chi^2})\cdot \mathbb E_{s_1\sim\rho}\textnormal{Var}_{\mu}\left[\sum_{h=1}^H\log \frac{\pi_h(a_h|s_h)}{\pi^*_h(a_h|s_h)}\Big| s_1\right]}.
    \end{equation*}
\end{lemma}

\begin{proof}
    Let $\delta_h(s_h,a_h):= \log\frac{\pi_h(a_h|s_h)}{\pi^*_h(a_h|s_h)}$, so that $\sum_{h=1}^H\delta_h(s_h,a_h)=\log\frac{\mathbb P^\pi(\tau|s_1)}{\mathbb P^{\pi^*}(\tau|s_1)}$. We decompose
    \begin{equation*}
    \begin{aligned}
        &\mathbb E_{s_1\sim\rho}\left[\textnormal{KL}\left(\mathbb P^\pi(\cdot|s_1)\|\mathbb P^{\pi^*}(\cdot|s_1)\right)\right]\\
        &=\mathbb E_{\pi}\left[\sum_{h=1}^H\delta_h(s_h,a_h)\right]\\
        &=\mathbb E_{\pi^*}\left[\sum_{h=1}^H\delta_h(s_h,a_h)\right]+\left(\mathbb E_{\pi}\left[\sum_{h=1}^H\delta_h(s_h,a_h)\right]-\mathbb E_{\mu}\left[\sum_{h=1}^H\delta_h(s_h,a_h)\right]\right)\\
        &\quad-\left(\mathbb E_{\pi^*}\left[\sum_{h=1}^H\delta_h(s_h,a_h)\right]-\mathbb E_{\mu}\left[\sum_{h=1}^H\delta_h(s_h,a_h)\right]\right)\\
        &\leq \left|\mathbb E_{\pi}\left[\sum_{h=1}^H\delta_h(s_h,a_h)\right]-\mathbb E_{\mu}\left[\sum_{h=1}^H\delta_h(s_h,a_h)\right]\right|\\
        &\quad+\left|\mathbb E_{\pi^*}\left[\sum_{h=1}^H\delta_h(s_h,a_h)\right]-\mathbb E_{\mu}\left[\sum_{h=1}^H\delta_h(s_h,a_h)\right]\right|,
    \end{aligned}
    \end{equation*}
    where the inequality is because $\mathbb E_{\pi^*}[\sum_{h=1}^H\delta_h(s_h,a_h)|s_1]=-\textnormal{KL}(\mathbb P^{\pi^*}(\cdot|s_1)\|\mathbb P^{\pi}(\cdot|s_1))\leq 0$. To bound the two remaining terms, define $g_1(s_1,a_1):=\delta_1(s_1,a_1)-\mathbb E_\mu[\sum_{h=1}^H\delta_h(s_h,a_h)|s_1]$ and $g_h:=\delta_h$ for $h\geq 2$. Then for $\nu\in\{\pi,\pi^*\}$,
    \begin{equation*}
    \begin{aligned}
        &\left(\mathbb E_{\nu}\left[\sum_{h=1}^H\delta_h(s_h,a_h)\right]-\mathbb E_{\mu}\left[\sum_{h=1}^H\delta_h(s_h,a_h)\right]\right)^2\\
        &=\left(\mathbb E_{\nu}\left[\sum_{h=1}^H g_h(s_h,a_h)\right]\right)^2\\
        &\leq \left(1+\sum_{h=1}^H\chi^2(d^\nu_h\|d^\mu_h)\right)\mathbb E_{\mu}\left[\left(\sum_{h=1}^H g_h(s_h,a_h)\right)^2\right]\\
        &\leq (1+C_{\chi^2})\cdot\mathbb E_{s_1\sim\rho}\textnormal{Var}_\mu\left[\sum_{h=1}^H\delta_h(s_h,a_h)\Big|s_1\right],
    \end{aligned}
    \end{equation*}
    where the first equation is because $s_1\sim\rho$ under both $\nu$ and $\mu$; the first inequality is by \cref{lemma:change_of_trajectory} with base measure $\mu$; the second inequality is because $\mathbb E_\mu[\sum_{h=1}^H g_h(s_h,a_h)|s_1]=0$ and $\nu\in\Pi$. Applying this bound with $\nu=\pi$ and $\nu=\pi^*$ to the two terms in the previous inequality proves the claim.
\end{proof}

Finally, we prove the concentration for the empirical objective $\mathcal E_{D,\textnormal{VM}}$ in \cref{eq:empirical_VM}.
\begin{lemma}\label{lemma:excess_risk_concentration}\textnormal{(Improved Excess Risk)}
Assume $|R^*_{\pi,\beta}(\tau)|\leq R_{\max}$ for any $\pi\in\Pi$ and any $\tau$ in the support of $\mathbb P^{\pi_{\textnormal{ref}}}(\cdot|s_1)$, $\pi^*\in\Pi$ and $|\Pi|<\infty$,
    let $\hat \pi\in \argmin_{\pi\in\Pi}\mathcal E_{D,\textnormal{VM}}(\pi)$, then it holds, w.p. at least $1-\zeta$
    \begin{equation*}
        \mathcal E_{\mu,\textnormal{VM}}(\hat \pi)\leq \frac{16}{3}\frac{R_{\max}^2\log(|\Pi|/\zeta)}{n}
    \end{equation*}
\end{lemma}

\begin{proof}
    Define,
    \begin{equation*}
        \ell_i(\pi):=\frac{1}{m-1}\sum_{j=1}^m \left(R_{\pi, \beta}^*(\tau^{ij})-\frac{1}{m}\sum_{k=1}^m R^*_{\pi,\beta}(\tau^{ik})\right)^2.
    \end{equation*}
    Notice that we should also have $\mathbb E_\mu[\ell_i(\pi)]=\mathcal E_{\mu,\textnormal{VM}}(\pi)$ and $\mathcal E_{D,\textrm{VM}}(\pi)=\frac{1}{n}\sum_{i=1}^n \ell_i(\pi)$. Since $|R^*_{\pi,\beta}(\tau^{ij})|\leq R_{\max}$ and $m\geq 2$, we have $\ell_i(\pi)\leq \frac{m}{m-1}R_{\max}^2\leq 2R_{\max}^2$, and therefore
    \begin{equation*}
        \textrm{Var}(\ell_i(\pi))\leq \mathbb E[\ell_i(\pi)^2]\leq 2R_{\max}^2\mathbb E[\ell_i(\pi)]=2R_{\max}^2 \mathcal E_{\mu,\textrm{VM}}(\pi).
    \end{equation*}
    Now by Bernstein's inequality and union bound, we have w.p. at least $1-\zeta$, for any $\pi\in \Pi$, it holds,
    \begin{equation*}
    \begin{aligned}
        \mathcal E_{\mu,\textnormal{VM}}(\pi)- \mathcal E_{D,\textrm{VM}}(\pi)&\leq \sqrt{\frac{2\textrm{Var}(\ell_i(\pi))\log(|\Pi|/\zeta)}{n}}+\frac{2R^2_{\max}\log(|\Pi|/\zeta)}{3n}\\
        &\leq \sqrt{\frac{4R_{\max}^2 \mathcal E_{\mu,\textrm{VM}}(\pi)\log(|\Pi|/\zeta)}{n}}+\frac{2R^2_{\max}\log(|\Pi|/\zeta)}{3n}\\
        &\leq \frac{1}{2}\mathcal E_{\mu,\textrm{VM}}(\pi)+\frac{2R_{\max}^2\log(|\Pi|/\zeta)}{n}+\frac{2R^2_{\max}\log(|\Pi|/\zeta)}{3n},
    \end{aligned}
    \end{equation*}
    where the last inequality is by $2\sqrt{ab}\leq a+b$. Rearranging gives $\mathcal E_{\mu,\textrm{VM}}(\pi)\leq 2\mathcal E_{D,\textrm{VM}}(\pi)+\frac{16}{3}\frac{R_{\max}^2\log(|\Pi|/\zeta)}{n}$ for any $\pi\in\Pi$, in particular for $\hat \pi$. Finally, by \cref{eq:trajectory_reward_KL_optimality}, $R^*_{\pi^*,\beta}(\tau^{ij})=V^*_{\beta,1}(s_1^i)$ for all $j$, hence $\ell_i(\pi^*)=0$ and $\mathcal E_{D,\textrm{VM}}(\hat \pi)\leq \mathcal E_{D,\textrm{VM}}(\pi^*)=0$, which proves the claim.
\end{proof}

\paragraph{Proof of \cref{thm:bound_all_policy_coverage}}
Now we use the above lemma to prove \cref{thm:bound_all_policy_coverage}.

\begin{proof}
     Under \cref{assume:all_policy_state_action_coverage}, $|R^*_{\pi,\beta}(\tau)|\leq H+B$ for any $\pi\in\Pi$ and any $\tau$ in the support of $\mathbb P^{\pi_{\textnormal{ref}}}(\cdot|s_1)$. Then we have, w.p. at least $1-\zeta$,
     \begin{equation*}
         \begin{aligned}
             J_\beta(\pi^*)-J_\beta(\hat \pi)&=\beta\mathbb E_{\hat \pi}\left[\sum_{h=1}^H \textrm{KL}(\hat \pi_h(\cdot|s_h)\|\pi_h^*(\cdot|s_h))\right]\\
             &=\beta\mathbb E_{s_1\sim\rho}\left[\textnormal{KL}\left(\mathbb P^{\hat \pi}(\cdot|s_1)\|\mathbb P^{\pi^*}(\cdot|s_1)\right)\right]\\
             &\leq 2\beta\sqrt{(1+C_{\chi^2})\cdot\mathbb E_{s_1\sim\rho}\textrm{Var}_{\mu}\left[\sum_{h=1}^H\log\frac{\hat \pi_h(a_h|s_h)}{\pi^*_h(a_h|s_h)}\Big | s_1\right]}\\
             &=2\sqrt{(1+C_{\chi^2})\cdot\mathcal E_{\mu,\textnormal{VM}}(\hat \pi)}\\
             &\lesssim (H+B)\sqrt{\frac{(1+C_{\chi^2})\log(|\Pi|/\zeta)}{n}},
         \end{aligned}
     \end{equation*}
     where the first equation is by \cref{lemma:reward_to_KL}, the second equation is by the chain rule of KL divergence, the first inequality is by \cref{lemma:KL_to_variance}, the next equation is because in \cref{eq:trajectory_reward_logratio_identity}, $V_{\beta,1}^*(s_1)$ is a constant given $s_1$, so that $\beta^2\mathbb E_{s_1\sim\rho}\textrm{Var}_{\mu}\left[\sum_{h=1}^H\log\frac{\hat \pi_h(a_h|s_h)}{\pi^*_h(a_h|s_h)}\Big | s_1\right]=\mathbb E_{s_1\sim\rho}\textrm{Var}_\mu\left[R^*_{\hat \pi,\beta}(\tau)\Big| s_1\right]=\mathcal E_{\mu,\textnormal{VM}}(\hat \pi)$, the last inequality is by \cref{lemma:excess_risk_concentration} with $R_{\max}=H+B$.
\end{proof}

\section{Proofs from Section 5}\label{app:proofs_sec5}

\subsection{Guarantee from Max-Min Game}\label{sec:maxmin_guarantee}
In this subsection, we show that solving the empirical version of \cref{eq:bandit_like_maxmin_game} gives a single-policy coverage guarantee (\cref{thm:guarantee_from_max_min_game}). In the next subsection, we show that the shift in the baseline is equivalent to solving such a max-min game. The reward-induced policy and \cref{lemma:KL_optimal_trajectory_sol,lemma:s1_shift_invariance} below are shared by both subsections.

Let us formally define the ingredients. We define
\begin{equation*}
    \mathcal P:=\left\{\mathbb P(\cdot|s_1): \textnormal{supp}(\mathbb P(\cdot|s_1))\subseteq \textnormal{supp}\left(\mathbb P^{\pi_{\textnormal{ref}}}(\cdot|s_1)\right),\ \forall s_1\in \textnormal{supp}(\rho)\right\}
\end{equation*}
as the set of conditional trajectory distributions supported on $\pi_{\textnormal{ref}}$, so that $\mathbb P^\pi\in\mathcal P$ for any policy $\pi$ supported on $\pi_{\textnormal{ref}}$. Throughout this section, $\mathcal R$ denotes a class of additive reward models $R(\tau)=\sum_{h=1}^H r_h(s_h,a_h)$, and $\|R\|_\infty:=\sup_\tau|R(\tau)|$. For $R\in\mathcal R$ and $\mathbb P\in \mathcal P$, let us define,
\begin{equation*}
     \phi(\mathbb P, R):=\mathcal L_\mu(\mathbb P, R)+\eta\cdot {\mathcal E}_\mu(R).
\end{equation*}
For a dataset $D$ we defined before, define the sample version of the above objectives as,
\begin{equation*}
    \begin{aligned}
        \mathcal L_D(\mathbb P,R):=& \frac{1}{n}\sum_{i=1}^n\Bigg[\mathbb E_{\tau\sim \mathbb P(\cdot|s_1^i)}[R(\tau)]-\frac{1}{m}\sum_{j=1}^m R(\tau^{ij})\\
        &\qquad\qquad-\beta \left(\textrm{KL}\left(\mathbb P(\cdot|s_1^{i})\|\mathbb P^{\pi_{\textrm{ref}}}(\cdot|s_1^i)\right)-\textrm{KL}\left(\mathbb P^\mu(\cdot|s_1^i)\|\mathbb P^{\pi_{\textrm{ref}}}(\cdot|s_1^i)\right)\right)\Bigg],\\
        {\mathcal E}_D(R):=&\frac{1}{n}\sum_{i=1}^n\frac{1}{m-1}\sum_{j=1}^m\Bigg\{\Big[\big(R^*(\tau^{ij})-R(\tau^{ij})\big)-\frac{1}{m}\sum_k{\big(R^*(\tau^{ik})-R(\tau^{ik})\big)}\Big]^2\Bigg\}\\
        \phi_D(\mathbb P,R):=&\mathcal L_D(\mathbb P, R)+\eta\cdot {\mathcal E}_D(R)
    \end{aligned}
\end{equation*}
Also we would slightly abuse the notation and denote $\phi(\pi,R)$ and $\phi_D(\pi, R)$ to denote $\phi(\mathbb P^\pi, R)$ and $\phi_D(\mathbb P^\pi, R)$ respectively.

The term $\beta\textnormal{KL}(\mathbb P^\mu(\cdot|s_1)\|\mathbb P^{\pi_{\textnormal{ref}}}(\cdot|s_1))$ in $\mathcal L_\mu$ and $\mathcal L_D$ depends on neither $\mathbb P$ nor $R$. For the concentration analysis, we therefore write $\mathcal L^\circ_\mu(\mathbb P,R)$ and $\mathcal L^\circ_D(\mathbb P,R)$ for $\mathcal L_\mu(\mathbb P,R)$ and $\mathcal L_D(\mathbb P,R)$ with this term removed, and $\phi^\circ_D:=\mathcal L^\circ_D+\eta\mathcal E_D$. Then $\phi_D$ and $\phi^\circ_D$ have the same max-min solutions, $\mathcal L^\circ_\mu(\mathbb P^\pi,R^*)-\mathcal L^\circ_\mu(\mathbb P^{\pi'},R^*)=J_\beta(\pi)-J_\beta(\pi')$ for any policies $\pi,\pi'$, and $\mathcal L^\circ_\mu,\mathcal L^\circ_D$ satisfy \cref{lemma:s1_shift_invariance} below as well.

For each $R$, we define the $R$-induced policy through the following procedure.  Recursively, we define $V_{R,H+1}:=0$ and for $h=H,\ldots,1$ define
\begin{equation*}
\begin{aligned}
    Q_{R,h}(s,a)&:=r_h(s,a)+V_{R,h+1}(T_h(s,a)), \\
    V_{R,h}(s)&:=\beta\log\int_{\mathcal A}\pi_{\mathrm{ref},h}(a|s)
    \exp\left(\frac{Q_{R,h}(s,a)}{\beta}\right)da.
\end{aligned}
\end{equation*}
Then define,
\begin{equation}\label{eq:reward_induced_policy}
\begin{aligned}
     \pi_{R,h}(a|s):=&
    \pi_{\mathrm{ref},h}(a|s)
    \exp\left(\frac{Q_{R,h}(s,a)-V_{R,h}(s)}{\beta}\right)\\
    \Pi_\mathcal R := &\{\pi_R, R\in \mathcal R\}.
\end{aligned}
\end{equation}

\begin{lemma}\label{lemma:KL_optimal_trajectory_sol}
    For a given dataset, for any $R\in \mathcal R$, the maximizer of $\phi_D(\mathbb P, R)$ exists in $\mathcal P$. Moreover, any maximizer $\mathbb P_R$ must take the form,
    \begin{equation*}
    \begin{aligned}
        \mathbb P_R(\tau|s_1)&=\mathbb P^{\pi_{\textnormal{ref}}}(\tau|s_1)\exp\left(\frac{R(\tau)}{\beta}\right)/Z_R(s_1)\\
        &=\mathbb P^{\pi_R}(\tau|s_1), \quad \forall s_1\in\{s_1^i,i\in[n]\}
    \end{aligned}
    \end{equation*}
    where $Z_R(s_1):=\int \mathbb P^{\pi_{\textnormal{ref}}}(\tau|s_1)\exp(R(\tau)/\beta)d\tau$ is the normalizing constant, and $\pi_R$ is defined in \cref{eq:reward_induced_policy}. In particular, $\beta\log Z_R(s_1)=R(\tau)-\beta\log\frac{\mathbb P^{\pi_R}(\tau|s_1)}{\mathbb P^{\pi_{\textnormal{ref}}}(\tau|s_1)}$ for all $\tau\in\textnormal{supp}(\mathbb P^{\pi_{\textnormal{ref}}}(\cdot|s_1))$.
\end{lemma}

\begin{proof}
     \begin{equation*}
     \begin{aligned}
        &\max_{\mathbb P\in\mathcal P}\phi_D(\mathbb P, R)
        \\
        &=\max_{\mathbb P\in \mathcal P}\Bigg\{\frac{1}{n}\sum_{i=1}^n\mathbb E_{\tau\sim \mathbb P(\cdot|s_1^{i})}\left[R(\tau)-\beta \textnormal{KL}(\mathbb P(\cdot|s_1^i)\|\mathbb P^{\pi_{\textrm{ref}}}(\cdot|s_1^i))\right]\Bigg\}+\textrm{const}\\
        &=\max_{\mathbb P\in \mathcal P}\Bigg\{\frac{1}{n}\sum_{i=1}^n\mathbb E_{\tau\sim\mathbb P(\cdot|s_1^i)}\left(R(\tau)-\beta\log\frac{\mathbb P(\tau|s_1^i)}{\mathbb P^{\pi_{\textrm{ref}}}(\tau|s_1^i)}\right)\Bigg\}+\textrm{const}\\
        &=\max_{\mathbb P\in \mathcal P}\Bigg\{\frac{1}{n}\sum_{i=1}^n\beta\log Z_R(s_1^i)-\frac{1}{n}\sum_{i=1}^n\beta\cdot\textnormal{KL}\left(\mathbb P(\cdot|s_1^i)\Big\|\frac{\mathbb P^{\pi_{\textrm{ref}}}(\cdot|s_1^i)\cdot\exp(\beta^{-1}R(\cdot))}{Z_R(s_1^i)}\right)\Bigg\}+\textrm{const},
    \end{aligned}
    \end{equation*}
    where const collects the terms not depending on $\mathbb P$. Hence the maximizer on the data support must equal
    \begin{equation*}
        \mathbb P_R(\tau|s_1)=\frac{\mathbb P^{\pi_{\textrm{ref}}}(\tau|s_1)\cdot\exp(\beta^{-1}R(\tau))}{Z_R(s_1)}=\frac{\prod_{h=1}^H \pi_{\textrm{ref},h}(a_h|s_h)\exp(\beta^{-1}r_h(s_h,a_h))}{Z_R(s_1)}.
    \end{equation*}
    Now we prove that $\pi_R$ exactly induces this conditional trajectory distribution $\mathbb P_R$. Notice that,
    \begin{equation*}
    \begin{aligned}
        \mathbb P^{\pi_R}(\tau|s_1)&=\prod_{h=1}^H \pi_{R,h}(a_h|s_h)=\prod_{h=1}^H \pi_{\textrm{ref},h}(a_h|s_h)\exp\left(\frac{Q_{R,h}(s_h,a_h)-V_{R,h}(s_h)}{\beta}\right)\\
        &=\prod_{h=1}^H \pi_{\textrm{ref},h}(a_h|s_h)\exp\left(\frac{r_h(s_h,a_h)+V_{R,h+1}(s_{h+1})-V_{R,h}(s_h)}{\beta}\right)\\
        &=\prod_{h=1}^H\pi_{\textrm{ref},h}(a_h|s_h)\exp\left(\frac{R(\tau)-V_{R,1}(s_1)}{\beta}\right)=\mathbb P_R(\tau|s_1),
    \end{aligned}
    \end{equation*}
    where the third equation is by telescoping and $V_{R,H+1}=0$, and the last equation is because both sides are probability distributions with the same unnormalized density, so $\exp(V_{R,1}(s_1)/\beta)=Z_R(s_1)$. The last claim follows by taking $\beta\log$ of the closed form of $\mathbb P_R$.
\end{proof}

Next we record that all the objectives above, as well as $\pi_R$, are invariant to reward shifts that depend only on $s_1$.
\begin{lemma}\label{lemma:s1_shift_invariance}
    For any additive $R$ and $c:\mathcal S\rightarrow\mathbb R$, let $R^c(\tau):=R(\tau)+c(s_1)$. Then $R^c$ is additive, $\pi_{R^c}=\pi_R$, and for any $\mathbb P\in\mathcal P$ and dataset $D$,
    \begin{equation*}
        \mathcal L_\mu(\mathbb P,R^c)=\mathcal L_\mu(\mathbb P,R),\quad \mathcal E_\mu(R^c)=\mathcal E_\mu(R),\quad \mathcal L_D(\mathbb P,R^c)=\mathcal L_D(\mathbb P,R),\quad \mathcal E_D(R^c)=\mathcal E_D(R).
    \end{equation*}
\end{lemma}
\begin{proof}
    Additivity holds by absorbing $c(s_1)$ into $r_1$. In $\mathcal L_\mu$ and $\mathcal L_D$, the two expectations of $R$ are taken over trajectories with the same initial state, so $c(s_1)$ cancels; $\mathcal E_\mu$ and $\mathcal E_D$ are conditional variances given $s_1$ and are unaffected by adding $c(s_1)$. Finally, the shift enters $Q_{R,1}$ and $V_{R,1}$ identically and cancels in \cref{eq:reward_induced_policy}.
\end{proof}

The following result is a fast-rate concentration on the excess risk of the variance minimization loss. For $\epsilon>0$, let $N_\infty(\mathcal R,\epsilon)$ denote the $\epsilon$-covering number of $\mathcal R$ under $\|\cdot\|_\infty$, with cover elements taken in $\mathcal R$.
\begin{lemma}\label{lemma:variance_loss_concentration}
    Let us define,
    \begin{equation*}
    \begin{aligned}
        e_i(R)&:=\frac{1}{m-1}\sum_{j=1}^m\Big[\big(R^*(\tau^{ij})-R(\tau^{ij})\big)-\frac{1}{m}\sum_{k=1}^m\big(R^*(\tau^{ik})-R(\tau^{ik})\big)\Big]^2,
    \end{aligned}
    \end{equation*}
    so that ${\mathcal E}_D(R)=\frac{1}{n}\sum_{i=1}^n e_i(R)$ and $\mathbb E_\mu[e_i(R)]={\mathcal E}_\mu(R)=\mathbb E_{s_1\sim \rho}\textnormal{Var}_\mu\left[R^*(\tau)-R(\tau)\big|s_1\right]$.
    Assume $\sup_{R\in\mathcal R}\|R\|_\infty\leq B$. Then w.p. at least $1-\zeta$, for any $R\in \mathcal R$,
    \begin{equation*}
        -{\mathcal E}_D(R)\leq -\frac{1}{2}{\mathcal E}_\mu(R)+\frac{32}{3}\frac{(H+B)^2\log(N_\infty(\mathcal R,1/n)/\zeta)}{n}+\frac{48(H+B)}{n}.
    \end{equation*}
\end{lemma}

\begin{proof}
    Since $R^*(\tau)\in[0,H]$ and $|R(\tau)|\leq B$, the values $R^*(\tau^{ij})-R(\tau^{ij})$ lie in $[-(H+B),H+B]$, hence $0\leq e_i(R)\leq \frac{m}{m-1}(2(H+B))^2\leq 8(H+B)^2$ and $\textnormal{Var}(e_i(R))\leq 8(H+B)^2\mathbb E[e_i(R)]=8(H+B)^2\mathcal E_\mu(R)$. Let $\mathcal R_{1/n}\subseteq\mathcal R$ be a $1/n$-cover of $\mathcal R$ with $|\mathcal R_{1/n}|\leq N_\infty(\mathcal R,1/n)$. By Bernstein's inequality and a union bound over $\mathcal R_{1/n}$, w.p. at least $1-\zeta$, for any $\tilde R\in\mathcal R_{1/n}$,
    \begin{equation*}
    \begin{aligned}
        &\mathcal E_\mu(\tilde R)-\mathcal E_D(\tilde R)\\
        &\leq \sqrt{\frac{16(H+B)^2\mathcal E_\mu(\tilde R)\log(N_\infty(\mathcal R,1/n)/\zeta)}{n}}+\frac{8(H+B)^2\log(N_\infty(\mathcal R,1/n)/\zeta)}{3n}\\
        &\leq \frac{1}{2}\mathcal E_\mu(\tilde R)+\frac{32}{3}\frac{(H+B)^2\log(N_\infty(\mathcal R,1/n)/\zeta)}{n},
    \end{aligned}
    \end{equation*}
    where the second inequality is by $2\sqrt{ab}\leq a+b$. Now for any $R\in\mathcal R$, pick $\tilde R\in\mathcal R_{1/n}$ with $\|R-\tilde R\|_\infty\leq 1/n$. Writing $z_j:=R^*(\tau^{ij})-R(\tau^{ij})$, $\tilde z_j:=R^*(\tau^{ij})-\tilde R(\tau^{ij})$, and $\bar z,\bar{\tilde z}$ for their averages over $j$, we have $|z_j-\bar z|,|\tilde z_j-\bar{\tilde z}|\leq 4(H+B)$ and $|(z_j-\bar z)-(\tilde z_j-\bar{\tilde z})|\leq 2/n$, so that
    \begin{equation*}
        |e_i(R)-e_i(\tilde R)|\leq \frac{1}{m-1}\sum_{j=1}^m\left|(z_j-\bar z)^2-(\tilde z_j-\bar{\tilde z})^2\right|\leq \frac{m}{m-1}\cdot\frac{16(H+B)}{n}\leq \frac{32(H+B)}{n},
    \end{equation*}
    and the same bound holds for $|\mathcal E_D(R)-\mathcal E_D(\tilde R)|$ and $|\mathcal E_\mu(R)-\mathcal E_\mu(\tilde R)|$. Combining,
    \begin{equation*}
    \begin{aligned}
        -\mathcal E_D(R)&\leq -\mathcal E_D(\tilde R)+\frac{32(H+B)}{n}\\
        &\leq -\frac{1}{2}\mathcal E_\mu(\tilde R)+\frac{32}{3}\frac{(H+B)^2\log(N_\infty(\mathcal R,1/n)/\zeta)}{n}+\frac{32(H+B)}{n}\\
        &\leq -\frac{1}{2}\mathcal E_\mu(R)+\frac{32}{3}\frac{(H+B)^2\log(N_\infty(\mathcal R,1/n)/\zeta)}{n}+\frac{48(H+B)}{n}.
    \end{aligned}
    \end{equation*}
\end{proof}

\begin{lemma}\label{lemma:uniform_convergence_concentration}
    Assume $\beta\leq 1$ and $\sup_{R\in\mathcal R}\|R\|_\infty\leq B$. Then for any fixed policy $\pi$ supported on $\pi_{\textnormal{ref}}$ with $\sup_{s_1\in \textnormal{supp}(\rho)}\beta \textnormal{KL}(\mathbb P^\pi(\cdot|s_1)\|\mathbb P^{\pi_{\textnormal{ref}}}(\cdot|s_1))\leq B$, w.p. at least $1-\zeta$,
    \begin{equation*}
    \begin{aligned}
        \sup_{R\in \mathcal R}\left|\mathcal L^\circ_\mu(\mathbb P^\pi,R)-\mathcal L^\circ_D(\mathbb P^\pi,R)\right|&\lesssim  (H+B)\sqrt{\frac{\log(N_\infty(\mathcal R,\beta/n)/\zeta)}{n}}, \\
       \sup_{R\in \mathcal R}\left|\mathcal L^\circ_\mu(\mathbb P^{\pi_R},R^*)-\mathcal L^\circ_D(\mathbb P^{\pi_R},R^*)\right|&\lesssim  (H+B)\sqrt{\frac{\log(N_\infty(\mathcal R,\beta/n)/\zeta)}{n}}.
    \end{aligned}
    \end{equation*}
\end{lemma}

\begin{proof}
    For any $(\mathbb P,R)$, define
    \begin{equation*}
    \begin{aligned}
        \ell_i(\mathbb P, R):=&\mathbb E_{\tau\sim \mathbb P(\cdot|s_1^i)}\left[R(\tau)\right]-\frac{1}{m}\sum_{j=1}^m R(\tau^{ij})-\beta\textrm{KL}\left(\mathbb P(\cdot|s_1^i)\|\mathbb P^{\pi_{\textrm{ref}}}(\cdot|s_1^i)\right),
    \end{aligned}
    \end{equation*}
    so that $\mathcal L^\circ_D(\mathbb P,R)=\frac{1}{n}\sum_{i=1}^n\ell_i(\mathbb P,R)$ and $\mathbb E_\mu[\ell_i(\mathbb P,R)]=\mathcal L^\circ_\mu(\mathbb P,R)$. Under the assumptions, $|\ell_i(\mathbb P^\pi,R)|\lesssim H+B$ for any $R\in\mathcal R$. Moreover, by \cref{lemma:KL_optimal_trajectory_sol}, $\beta\textrm{KL}(\mathbb P^{\pi_R}(\cdot|s_1)\|\mathbb P^{\pi_{\textrm{ref}}}(\cdot|s_1))=\mathbb E_{\mathbb P^{\pi_R}(\cdot|s_1)}[R(\tau)]-\beta\log Z_R(s_1)$, so that
    \begin{equation*}
    \begin{aligned}
        \ell_i(\mathbb P^{\pi_R},R^*)=&\mathbb E_{\tau\sim\mathbb P^{\pi_R}(\cdot|s_1^i)}\left[R^*(\tau)-R(\tau)\right]+\beta\log Z_R(s_1^i)-\frac{1}{m}\sum_{j=1}^m R^*(\tau^{ij}),
    \end{aligned}
    \end{equation*}
    and $|\beta\log Z_R(s_1)|\leq\|R\|_\infty$ gives $|\ell_i(\mathbb P^{\pi_R},R^*)|\lesssim H+B$.

    Let $\mathcal R_{\beta/n}\subseteq\mathcal R$ be a $\beta/n$-cover of $\mathcal R$ under $\|\cdot\|_\infty$ with $|\mathcal R_{\beta/n}|\leq N_\infty(\mathcal R,\beta/n)$. For each fixed $\tilde R\in\mathcal R_{\beta/n}$, the variables $\{\ell_i(\mathbb P^\pi,\tilde R)\}_{i\in[n]}$ are i.i.d.\ with mean $\mathcal L^\circ_\mu(\mathbb P^\pi,\tilde R)$ and bounded by $O(H+B)$, and likewise for $\{\ell_i(\mathbb P^{\pi_{\tilde R}},R^*)\}_{i\in[n]}$. By Hoeffding's inequality and a union bound over $\mathcal R_{\beta/n}$, w.p. at least $1-\zeta$,
    \begin{equation*}
    \begin{aligned}
        \max_{\tilde R\in\mathcal R_{\beta/n}}\left|\mathcal L^\circ_\mu(\mathbb P^\pi,\tilde R)-\mathcal L^\circ_D(\mathbb P^\pi,\tilde R)\right|&\lesssim (H+B)\sqrt{\frac{\log(N_\infty(\mathcal R,\beta/n)/\zeta)}{n}},\\
        \max_{\tilde R\in\mathcal R_{\beta/n}}\left|\mathcal L^\circ_\mu(\mathbb P^{\pi_{\tilde R}},R^*)-\mathcal L^\circ_D(\mathbb P^{\pi_{\tilde R}},R^*)\right|&\lesssim (H+B)\sqrt{\frac{\log(N_\infty(\mathcal R,\beta/n)/\zeta)}{n}}.
    \end{aligned}
    \end{equation*}
    Now fix $R\in\mathcal R$ and pick $\tilde R\in\mathcal R_{\beta/n}$ with $\|R-\tilde R\|_\infty\leq\beta/n\leq 1/n$. For every $i$ and every realization of the data, $|\ell_i(\mathbb P^\pi,R)-\ell_i(\mathbb P^\pi,\tilde R)|\leq 2\|R-\tilde R\|_\infty\leq 2/n$. For the second bound, $\exp((R(\tau)-\tilde R(\tau))/\beta)\in[e^{-1/n},e^{1/n}]$ pointwise, and integrating against $\mathbb P^{\pi_{\textnormal{ref}}}(\cdot|s_1)$ gives $Z_{\tilde R}(s_1)/Z_R(s_1)\in[e^{-1/n},e^{1/n}]$ as well. Hence, by \cref{lemma:KL_optimal_trajectory_sol}, the density ratio $\frac{\mathbb P^{\pi_R}(\tau|s_1)}{\mathbb P^{\pi_{\tilde R}}(\tau|s_1)}=\exp\left(\frac{R(\tau)-\tilde R(\tau)}{\beta}\right)\frac{Z_{\tilde R}(s_1)}{Z_R(s_1)}$ lies in $[e^{-2/n},e^{2/n}]$, so $\|\mathbb P^{\pi_R}(\cdot|s_1)-\mathbb P^{\pi_{\tilde R}}(\cdot|s_1)\|_1\leq e^{2/n}-1\leq 4/n$ for $n\geq 2$, and by the expression of $\ell_i(\mathbb P^{\pi_R},R^*)$ above,
    \begin{equation*}
    \begin{aligned}
        &|\ell_i(\mathbb P^{\pi_R},R^*)-\ell_i(\mathbb P^{\pi_{\tilde R}},R^*)|\\
        &\leq \|R-\tilde R\|_\infty+\|R^*-\tilde R\|_\infty\cdot\|\mathbb P^{\pi_R}(\cdot|s_1^i)-\mathbb P^{\pi_{\tilde R}}(\cdot|s_1^i)\|_1+|\beta\log Z_R(s_1^i)-\beta\log Z_{\tilde R}(s_1^i)|\\
        &\leq \frac{1}{n}+\frac{4(H+B)}{n}+\frac{1}{n},
    \end{aligned}
    \end{equation*}
    where the first inequality writes $\mathbb E_{\pi_R}[R^*-R]-\mathbb E_{\pi_{\tilde R}}[R^*-\tilde R]=\mathbb E_{\pi_R}[\tilde R-R]+(\mathbb E_{\pi_R}-\mathbb E_{\pi_{\tilde R}})[R^*-\tilde R]$, and the last term is bounded by $\beta/n$ by the bound on $Z_{\tilde R}/Z_R$. Averaging over $i$ and taking expectations, the same two bounds hold for $\mathcal L^\circ_D$ and $\mathcal L^\circ_\mu$. Therefore, on the above event, for every $R\in\mathcal R$ the two deviations exceed those of its nearest $\tilde R$ by at most $4/n$ and $(4+8(H+B))/n$ respectively, which are dominated by the $O((H+B)\sqrt{\log(N_\infty(\mathcal R,\beta/n)/\zeta)/n})$ bound on the cover. This gives the claim.
\end{proof}

Now we provide the regret bound.
\begin{theorem}\label{thm:guarantee_from_max_min_game}
    Let $\beta\leq 1$ and $\eta=\sqrt{n}$. Assume \textnormal{(1)} $\sup_{R\in\mathcal R}\|R\|_\infty\leq B$; \textnormal{(2)} there exists $\bar R\in\mathcal R$ such that $\bar R(\tau)-R^*(\tau)$ depends on $\tau$ only through $s_1$; \textnormal{(3)} there exists $\hat \pi\in\Pi_{\mathcal R}$ such that
    \begin{equation*}
        \mathbb P^{\hat \pi}\in \argmax_{\mathbb P\in\mathcal P}\inf_{R\in\mathcal R}\phi_D(\mathbb P, R).
    \end{equation*}
    Then for any comparator policy $\pi$ supported on $\pi_{\textnormal{ref}}$ with $\sup_{s_1\in \textnormal{supp}(\rho)}\beta \textnormal{KL}(\mathbb P^\pi(\cdot|s_1)\|\mathbb P^{\pi_{\textnormal{ref}}}(\cdot|s_1))\leq B$, w.p. at least $1-\zeta$,
    \begin{equation*}
        J_\beta(\pi)-J_\beta(\hat \pi)\lesssim \left(\sum_{h=1}^H\chi^2(d_h^\pi\|d_h^\mu)+(H+B)^2\log\frac{N_\infty(\mathcal R,\beta/n)}{\zeta}\right)\frac{1}{\sqrt n}.
    \end{equation*}
\end{theorem}

\begin{proof}
    Recall that $\mathcal L^\circ_\mu,\mathcal L^\circ_D,\phi^\circ_D$ are $\mathcal L_\mu,\mathcal L_D,\phi_D$ with the term $\beta\textnormal{KL}(\mathbb P^\mu(\cdot|s_1)\|\mathbb P^{\pi_{\textnormal{ref}}}(\cdot|s_1))$ removed, which depends on neither $\mathbb P$ nor $R$; hence $\mathbb P^{\hat\pi}$ is also a max-min solution of $\phi^\circ_D$, and $J_\beta(\pi)-J_\beta(\hat\pi)=\mathcal L^\circ_\mu(\mathbb P^{\pi},R^*)-\mathcal L^\circ_\mu(\mathbb P^{\hat\pi},R^*)$. We decompose
    \begin{equation*}
        J_\beta(\pi)-J_\beta(\hat \pi)=\underbrace{\mathcal L^\circ_\mu(\mathbb P^\pi, R^*)-\inf_{R\in \mathcal R}\phi^\circ_D(\mathbb P^{\hat \pi},R)}_{:=\textrm{Term}(A)}+\underbrace{ \inf_{R\in \mathcal R}\phi^\circ_D(\mathbb P^{\hat \pi},R)-\mathcal L^\circ_\mu(\mathbb P^{\hat \pi},R^*)}_{:=\textrm{Term}(B)}.
    \end{equation*}
    For Term (A), we have
    \begin{equation*}
    \begin{aligned}
        \textrm{Term}(A)&\leq \mathcal L^\circ_\mu(\mathbb P^\pi, R^*)-\inf_{R\in\mathcal R}\phi^\circ_D(\mathbb P^{\pi},R)\\
        &=\sup_{R\in \mathcal R}\left\{\mathcal L^\circ_\mu(\mathbb P^\pi,R^*)-\mathcal L^\circ_\mu(\mathbb P^\pi, R)-\eta{\mathcal E}_D(R)+\mathcal L^\circ_\mu(\mathbb P^\pi,R)-\mathcal L^\circ_D(\mathbb P^\pi,R)\right\}\\
        &\leq \sup_{R\in \mathcal R}\left\{\mathcal L^\circ_\mu(\mathbb P^\pi,R^*)-\mathcal L^\circ_\mu(\mathbb P^\pi, R)-\eta{\mathcal E}_D(R)\right\}+\sup_{R\in\mathcal R}\left(\mathcal L^\circ_\mu(\mathbb P^\pi,R)-\mathcal L^\circ_D(\mathbb P^\pi,R)\right),
    \end{aligned}
    \end{equation*}
    where the first inequality is by the optimality of $\mathbb P^{\hat \pi}$ over $\mathbb P^\pi\in\mathcal P$. For the two supremum above, the second  is a uniform deviation that we control at the end; we now bound the first supremum. Define $g_1(s_1,a_1):=r^*_1(s_1,a_1)-r_1(s_1,a_1)-\mathbb E_\mu\left[\sum_{h=1}^H\left(r^*_h(s_h,a_h)-r_h(s_h,a_h)\right)\big|s_1\right]$ and $g_h:=r^*_h-r_h$ for $h\geq 2$, so that $\mathbb E_\mu[\sum_{h=1}^Hg_h(s_h,a_h)|s_1]=0$. Then
    \begin{equation*}
    \begin{aligned}
        &\mathcal L^\circ_\mu(\mathbb P^\pi,R^*)-\mathcal L^\circ_\mu(\mathbb P^\pi, R)\\
        &=\mathbb E_{\pi}\left[\sum_{h=1}^H\left(r^*_h(s_h,a_h)-r_h(s_h,a_h)\right)\right]-\mathbb E_{\mu}\left[\sum_{h=1}^H\left(r^*_h(s_h,a_h)-r_h(s_h,a_h)\right)\right]\\
        &=\mathbb E_{\pi}\left[\sum_{h=1}^Hg_h(s_h,a_h)\right]
        \leq \sqrt{\left(1+\sum_{h=1}^H\chi^2(d^\pi_h\|d^\mu_h)\right)\mathbb E_\mu\left[\Big(\sum_{h=1}^Hg_h(s_h,a_h)\Big)^2\right]}\\
        &=\sqrt{\left(1+\sum_{h=1}^H\chi^2(d^\pi_h\|d^\mu_h)\right)\mathcal E_\mu(R)},
    \end{aligned}
    \end{equation*}
    in which the second equation is because $s_1\sim\rho$ under both $\pi$ and $\mu$, the inequality is by \cref{lemma:change_of_trajectory}, and the last equation is because $\mathbb E_\mu[\sum_{h=1}^Hg_h(s_h,a_h)|s_1]=0$. Notice that, by \cref{lemma:variance_loss_concentration}, we also have w.p. at least $1-\zeta$, $-\eta\mathcal E_D(R)\leq-\frac{\eta}{2}\mathcal E_\mu(R)+\eta\epsilon_n$ for all $R\in\mathcal R$, where $\epsilon_n:=\frac{32}{3}\frac{(H+B)^2\log(N_\infty(\mathcal R,1/n)/\zeta)}{n}+\frac{48(H+B)}{n}$. Combining the two bounds,
    \begin{equation*}
    \begin{aligned}
        &\sup_{R\in \mathcal R}\left\{\mathcal L^\circ_\mu(\mathbb P^\pi,R^*)-\mathcal L^\circ_\mu(\mathbb P^\pi, R)-\eta{\mathcal E}_D(R)\right\}\\
        &\leq \sup_{R\in\mathcal R}\left\{\sqrt{\left(1+\sum_{h=1}^H\chi^2(d^\pi_h\|d^\mu_h)\right)\mathcal E_\mu(R)}-\frac{\eta}{2}\mathcal E_\mu(R)\right\}+\eta\epsilon_n\\
        &\leq \left(1+\sum_{h=1}^H\chi^2(d^\pi_h\|d^\mu_h)\right)\frac{1}{2\eta}+\eta\epsilon_n,
    \end{aligned}
    \end{equation*}
    where the last inequality is by $\sqrt{ax}-\frac{\eta}{2}x\leq\frac{a}{2\eta}$ for $x\geq 0$. Therefore,
    \begin{equation*}
        \textrm{Term}(A)\leq \left(1+\sum_{h=1}^H\chi^2(d^\pi_h\|d^\mu_h)\right)\frac{1}{2\eta}+\eta\epsilon_n+\sup_{R\in\mathcal R}\left(\mathcal L^\circ_\mu(\mathbb P^\pi,R)-\mathcal L^\circ_D(\mathbb P^\pi,R)\right).
    \end{equation*}
    For Term (B), we have
    \begin{equation*}
    \begin{aligned}
        \textrm{Term}(B)&\leq \phi^\circ_D(\mathbb P^{\hat \pi},\bar R)-\mathcal L^\circ_\mu(\mathbb P^{\hat \pi},R^*)\\
        &=\mathcal L^\circ_D(\mathbb P^{\hat \pi},R^*)+\eta\mathcal E_D(R^*)-\mathcal L^\circ_\mu(\mathbb P^{\hat \pi},R^*)\\
        &=\mathcal L^\circ_D(\mathbb P^{\hat \pi},R^*)-\mathcal L^\circ_\mu(\mathbb P^{\hat \pi},R^*)\\
        &\leq \sup_{R\in \mathcal R}\left|\mathcal L^\circ_D(\mathbb P^{\pi_R},R^*)-\mathcal L^\circ_\mu(\mathbb P^{\pi_R},R^*)\right|,
    \end{aligned}
    \end{equation*}
    where the first inequality takes $R=\bar R\in\mathcal R$ from assumption (2) in the infimum; the first equation is by \cref{lemma:s1_shift_invariance}, since $\bar R-R^*$ depends only on $s_1$; the second equation is because $\mathcal E_D(R^*)=0$; and the last inequality is by assumption (3). Finally, controlling the two uniform deviations by \cref{lemma:uniform_convergence_concentration}, using $N_\infty(\mathcal R,1/n)\leq N_\infty(\mathcal R,\beta/n)$ and plugging in $\eta=\sqrt n$, we obtain
    \begin{equation*}
    \begin{aligned}
        J_\beta(\pi)-J_\beta(\hat \pi)&\lesssim \left(1+\sum_{h=1}^H\chi^2(d_h^\pi\|d_h^\mu)\right)\frac{1}{\sqrt n}+\frac{(H+B)^2\log(N_\infty(\mathcal R,\beta/n)/\zeta)}{\sqrt n}+\frac{H+B}{\sqrt n}\\
        &\quad+(H+B)\sqrt{\frac{\log(N_\infty(\mathcal R,\beta/n)/\zeta)}{n}},
    \end{aligned}
    \end{equation*}
    which gives the claim after adjusting $\zeta$ by a constant factor for the union bound.
\end{proof}

\subsection{Connecting AsymRE to the Max-Min Game}\label{sec:proof_thm_single_policy}
In this subsection, we show that the SFT term in the AsymRE update is equivalent to an adversarial regularizer. The connection is established via
\begin{equation*}
    \textnormal{AsymRE}
    =
    \textnormal{GRPG}+\textnormal{SFT}
    \overset{\textnormal{proportional gradient}}{\sim}
    \textnormal{VM}+\textnormal{SFT}\overset{\textnormal{set equivalence}}{\sim} \textrm{max-min game}.
\end{equation*}
All of those equivalences are established at the sample level, with $\mathcal L_{D,\textrm{SFT+VM}}$ defined in \cref{eq:SFT+Vairance_minimize}. For a policy $\pi$ supported on $\pi_{\textnormal{ref}}$, define the policy-induced reward
\begin{equation}\label{eq:policy_induced_reward}
    R_\pi(\tau):=\beta\log\frac{\mathbb P^\pi(\tau|s_1)}{\mathbb P^{\pi_{\textnormal{ref}}}(\tau|s_1)}=\sum_{h=1}^H \beta\log\frac{\pi_h(a_h|s_h)}{\pi_{\textnormal{ref},h}(a_h|s_h)},\qquad \mathcal R_\Pi:=\{R_\pi:\pi\in\Pi\}.
\end{equation}
It inverts the map $R\mapsto\pi_R$ in \cref{eq:reward_induced_policy}, which is the trajectory-level analogue of the reparameterization in \citet{rafailov2023direct,liu2024provably}.
\begin{lemma}\label{lemma:policy_induced_reward}
    For any $\pi$ supported on $\pi_{\textnormal{ref}}$, $R_\pi$ is additive, $Z_{R_\pi}(s_1)=1$, and $\pi_{R_\pi}=\pi$. Consequently, for any policy class $\Pi$ supported on $\pi_{\textnormal{ref}}$, $\Pi_{\mathcal R_\Pi}=\Pi$, and since $\|R_\pi-R_{\pi'}\|_\infty=\beta \textnormal{dist}(\pi,\pi')$, $N_\infty(\mathcal R_\Pi,\beta\epsilon)=N(\Pi,\epsilon)$ for any $\epsilon>0$.
\end{lemma}
\begin{proof}
    Additivity is by the definition of $R_\pi$, whose step-$h$ component is $\beta\log\frac{\pi_h(a|s)}{\pi_{\textnormal{ref},h}(a|s)}$, and $Z_{R_\pi}(s_1)=\int \mathbb P^{\pi_{\textnormal{ref}}}(\tau|s_1)\frac{\mathbb P^\pi(\tau|s_1)}{\mathbb P^{\pi_{\textnormal{ref}}}(\tau|s_1)}d\tau=1$. For the recursion in \cref{eq:reward_induced_policy}, we show $V_{R_\pi,h}\equiv 0$ by backward induction. It holds for $h=H+1$ by definition, and if $V_{R_\pi,h+1}\equiv 0$, then
    \begin{equation*}
        Q_{R_\pi,h}(s,a)=\beta\log\frac{\pi_h(a|s)}{\pi_{\textnormal{ref},h}(a|s)},\qquad V_{R_\pi,h}(s)=\beta\log\int_{\mathcal A}\pi_{\textnormal{ref},h}(a|s)\frac{\pi_h(a|s)}{\pi_{\textnormal{ref},h}(a|s)}da=0.
    \end{equation*}
    Therefore $\pi_{R_\pi,h}(a|s)=\pi_{\textnormal{ref},h}(a|s)\exp(Q_{R_\pi,h}(s,a)/\beta)=\pi_h(a|s)$. The last claim is because $R_\pi(\tau)-R_{\pi'}(\tau)=\beta\log\frac{\mathbb P^\pi(\tau|s_1)}{\mathbb P^{\pi'}(\tau|s_1)}$.
\end{proof}

The key structural fact is that, under \cref{assum:geometric_mixture_closure}, $\mathcal R_\Pi$ is convex up to shifts depending only on $s_1$.
\begin{lemma}\label{lemma:geometric_mixture_convexity}
    Let $\Pi$ be supported on $\pi_{\textnormal{ref}}$ with $\sup_{\pi\in\Pi}\|R_\pi\|_\infty<\infty$, and satisfy \cref{assum:geometric_mixture_closure}. Then for any $R_1,R_2\in\mathcal R_\Pi$ and $\alpha\in[0,1]$, there exist $R_3\in\mathcal R_\Pi$ and $c:\mathcal S\rightarrow\mathbb R$ such that $\alpha R_1+(1-\alpha)R_2=R_3+c(s_1)$. Consequently, for any $\mathbb P\in\mathcal P$ and dataset $D$, $\phi_D(\mathbb P,R_3)\leq \alpha\phi_D(\mathbb P,R_1)+(1-\alpha)\phi_D(\mathbb P,R_2)$.
\end{lemma}
\begin{proof}
    Write $R_1=R_{\pi_1}$, $R_2=R_{\pi_2}$, and let $\pi_3\in\Pi$ be given by \cref{assum:geometric_mixture_closure} with normalizing constant $Z_\alpha(s_1)$, i.e., $\mathbb P^{\pi_3}(\tau|s_1)=\mathbb P^{\pi_1}(\tau|s_1)^{\alpha}\mathbb P^{\pi_2}(\tau|s_1)^{1-\alpha}/Z_\alpha(s_1)$. Then
    \begin{equation*}
        \alpha R_{\pi_1}(\tau)+(1-\alpha)R_{\pi_2}(\tau)=\beta\log\frac{\mathbb P^{\pi_1}(\tau|s_1)^\alpha\mathbb P^{\pi_2}(\tau|s_1)^{1-\alpha}}{\mathbb P^{\pi_{\textnormal{ref}}}(\tau|s_1)}=R_{\pi_3}(\tau)+\beta\log Z_\alpha(s_1).
    \end{equation*}
    For the second claim, $\phi_D(\mathbb P,R)$ is convex in $R$, since $\mathcal L_D(\mathbb P,R)$ is linear in $R$ and $\mathcal E_D(R)$ is a sample variance of $R^*-R$ and thus a convex quadratic in $R$. Therefore, by \cref{lemma:s1_shift_invariance}, $\phi_D(\mathbb P,R_3)=\phi_D(\mathbb P,\alpha R_1+(1-\alpha)R_2)\leq\alpha\phi_D(\mathbb P,R_1)+(1-\alpha)\phi_D(\mathbb P,R_2)$.
\end{proof}

\begin{lemma}[Minimax theorem {\citep{fan1953minimax}}]\label{lem:minimax_fan}
Let $\mathcal X$ be a nonempty compact topological space and $\mathcal Y$ a nonempty set (not necessarily topologized). Let $g:\mathcal X\times\mathcal Y\rightarrow \mathbb R$ be upper semicontinuous on $\mathcal X$ for every $y\in\mathcal Y$. Suppose that $g$ is concave-like on $\mathcal X$ and convex-like on $\mathcal Y$, i.e., for any $x_1,x_2\in\mathcal X$ and $\alpha\in[0,1]$, there exists $x_3\in\mathcal X$ such that $g(x_3,\cdot)\geq\alpha g(x_1,\cdot)+(1-\alpha)g(x_2,\cdot)$ on $\mathcal Y$, and for any $y_1,y_2\in\mathcal Y$ and $\alpha\in[0,1]$, there exists $y_3\in\mathcal Y$ such that $g(\cdot,y_3)\leq\alpha g(\cdot,y_1)+(1-\alpha)g(\cdot,y_2)$ on $\mathcal X$. Then
\begin{equation*}
    \max_{x\in\mathcal X}\inf_{y\in\mathcal Y} g(x,y)
    =
    \inf_{y\in\mathcal Y}\max_{x\in\mathcal X} g(x,y).
\end{equation*}
\end{lemma}

The following proposition shows that minimizing the SFT+VM objective over $\Pi$ is equivalent to solving the max-min game with the reward class $\mathcal R_\Pi$.
\begin{proposition}\label{prop:maxmin_closed_form}
    Let $\mathcal A$ be finite, and let $\Pi$ be supported on $\pi_{\textnormal{ref}}$ with $\sup_{\pi\in\Pi}\|R_\pi\|_\infty<\infty$, and satisfy \cref{assum:geometric_mixture_closure}. Given $\eta,\beta>0$, if $\hat \pi\in\argmin_{\pi\in \Pi}\mathcal L_{D,\textnormal{SFT+VM}}(\pi)$, then
    \begin{equation*}
        \mathbb P^{\hat \pi}\in \argmax_{\mathbb P\in\mathcal P}\inf_{R\in \mathcal R_\Pi}\phi_D(\mathbb P, R).
    \end{equation*}
\end{proposition}

\begin{proof}
    We first show that $R_{\hat\pi}$ minimizes $\max_{\mathbb P\in\mathcal P}\phi_D(\mathbb P,R)$ over $R\in\mathcal R_\Pi$. By \cref{lemma:KL_optimal_trajectory_sol}, for any additive $R$ the maximizer is $\mathbb P^{\pi_R}$ and $R(\tau)-\beta\log\frac{\mathbb P^{\pi_R}(\tau|s_1)}{\mathbb P^{\pi_{\textnormal{ref}}}(\tau|s_1)}=\beta\log Z_R(s_1)$, hence
    \begin{equation}\label{eq:maxmin_sft_vm_identity}
    \begin{aligned}
        \max_{\mathbb P\in\mathcal P}\phi_D(\mathbb P,R)&=\phi_D(\mathbb P^{\pi_R},R)\\
        &=\frac{1}{n}\sum_{i=1}^n\left[\beta\log Z_R(s_1^i)-\frac{1}{m}\sum_{j=1}^mR(\tau^{ij})\right]+\frac{\beta}{n}\sum_{i=1}^n\textrm{KL}\left(\mathbb P^\mu(\cdot|s_1^i)\|\mathbb P^{\pi_{\textrm{ref}}}(\cdot|s_1^i)\right)\\
        &\quad+\eta\cdot{\mathcal E}_D(R)\\
        &=\frac{\beta}{nm}\sum_{i=1}^n\sum_{j=1}^m\left[-\log \frac{\mathbb P^{\pi_R}(\tau^{ij}|s_1^i)}{\mathbb P^{\pi_{\text{ref}}}(\tau^{ij}|s_1^i)}\right]+\eta\cdot {\mathcal E}_D\left(\beta\log \frac{\mathbb P^{\pi_R}(\cdot|\cdot)}{\mathbb P^{\pi_{\text{ref}}}(\cdot|\cdot)}\right)+\textrm{const}\\
        &=\beta\cdot \frac{1}{nm}\sum_{i=1}^n\sum_{j=1}^m\left[-\sum_{h=1}^H\log \pi_{R,h}(a_h^{ij}|s_h^{ij})\right]+\eta\cdot\mathcal E_{D,\textrm{VM}}(\pi_R)+\textrm{const}\\
        &=\mathcal L_{D,\textrm{SFT+VM}}(\pi_R)+\textrm{const},
    \end{aligned}
    \end{equation}
    where the first equation is because $\mathbb E_{\mathbb P^{\pi_R}(\cdot|s_1^i)}[R(\tau)]-\beta\textrm{KL}(\mathbb P^{\pi_R}(\cdot|s_1^i)\|\mathbb P^{\pi_{\textrm{ref}}}(\cdot|s_1^i))=\beta\log Z_R(s_1^i)$; the second equation is by \cref{lemma:s1_shift_invariance} applied to the shift $\beta\log Z_R(s_1)$, with const collecting $\frac{\beta}{n}\sum_i\textrm{KL}(\mathbb P^\mu(\cdot|s_1^i)\|\mathbb P^{\pi_{\textrm{ref}}}(\cdot|s_1^i))+\frac{\beta}{nm}\sum_{i,j}\log\mathbb P^{\pi_{\textrm{ref}}}(\tau^{ij}|s_1^i)$, which does not depend on $R$; the third equation is because $R^*(\tau)-\beta\log\frac{\mathbb P^{\pi_R}(\tau|s_1)}{\mathbb P^{\pi_{\textrm{ref}}}(\tau|s_1)}=R^*_{\pi_R,\beta}(\tau)$ and the transition is deterministic.

    Now take $R=R_\pi$ for $\pi\in\Pi$ in \cref{eq:maxmin_sft_vm_identity}. Since $\pi_{R_\pi}=\pi$ by \cref{lemma:policy_induced_reward}, we get $\max_{\mathbb P\in\mathcal P}\phi_D(\mathbb P,R_\pi)=\mathcal L_{D,\textnormal{SFT+VM}}(\pi)+\textnormal{const}$ for every $\pi\in\Pi$. Taking $\pi=\hat\pi$ and using the optimality of $\hat\pi$ over $\Pi$,
    \begin{equation*}
        R_{\hat\pi}\in\argmin_{R\in\mathcal R_\Pi}\max_{\mathbb P\in\mathcal P}\phi_D(\mathbb P,R),
    \end{equation*}
    i.e., $R_{\hat\pi}$ solves the min-max problem. It remains to convert this into a max-min solution.

    Next we verify the conditions of \cref{lem:minimax_fan} with $\mathcal X=\mathcal P$, $\mathcal Y=\mathcal R_\Pi$, and $g=\phi_D$. Since $\phi_D$ depends on $\mathbb P$ only through $\{\mathbb P(\cdot|s_1^i)\}_{i\in[n]}$, and each $\mathbb P(\cdot|s_1^i)$ is a distribution over the finitely many trajectories from $s_1^i$ ($\mathcal A$ and $H$ are finite and transitions are deterministic), we may take $\mathcal X$ to be the set of such $n$-tuples, which is compact. For each $R$, $\phi_D(\cdot,R)$ is continuous and concave on this set, being linear in $\mathbb P$ minus $\beta$ times a KL divergence to $\mathbb P^{\pi_{\textnormal{ref}}}$, and $\phi_D$ is convex-like on $\mathcal R_\Pi$ by \cref{lemma:geometric_mixture_convexity}. Hence \cref{lem:minimax_fan} gives
    \begin{equation*}
        \max_{\mathbb P\in\mathcal P}\inf_{R\in\mathcal R_\Pi}\phi_D(\mathbb P,R)=\inf_{R\in\mathcal R_\Pi}\max_{\mathbb P\in\mathcal P}\phi_D(\mathbb P,R),
    \end{equation*}
    and the maximum on the left is attained by some $\hat{\mathbb P}\in\mathcal P$, since $\inf_{R\in\mathcal R_\Pi}\phi_D(\cdot,R)$ is upper semicontinuous on a compact set. Then
    \begin{equation*}
    \begin{aligned}
        0&=\left(\max_{\mathbb P\in\mathcal P}\phi_D(\mathbb P,R_{\hat\pi})-\inf_{R\in\mathcal R_\Pi}\max_{\mathbb P\in\mathcal P}\phi_D(\mathbb P,R)\right)\\
        &\quad+\left(\max_{\mathbb P\in\mathcal P}\inf_{R\in\mathcal R_\Pi}\phi_D(\mathbb P,R)-\inf_{R\in\mathcal R_\Pi}\phi_D(\hat{\mathbb P},R)\right)\\
        &=\max_{\mathbb P\in\mathcal P}\phi_D(\mathbb P,R_{\hat\pi})-\inf_{R\in\mathcal R_\Pi}\phi_D(\hat{\mathbb P},R)\\
        &=\left(\max_{\mathbb P\in\mathcal P}\phi_D(\mathbb P,R_{\hat\pi})-\phi_D(\hat{\mathbb P},R_{\hat\pi})\right)+\left(\phi_D(\hat{\mathbb P},R_{\hat\pi})-\inf_{R\in\mathcal R_\Pi}\phi_D(\hat{\mathbb P},R)\right),
    \end{aligned}
    \end{equation*}
    where in the first equation both brackets vanish, by the min-max optimality of $R_{\hat\pi}$ and the max-min optimality of $\hat{\mathbb P}$ respectively, and the second equation is by the minimax equality above. Both terms in the last line are nonnegative and hence zero, so $\hat{\mathbb P}$ maximizes $\phi_D(\cdot,R_{\hat\pi})$. By \cref{lemma:KL_optimal_trajectory_sol,lemma:policy_induced_reward}, this maximizer equals $\mathbb P^{\pi_{R_{\hat\pi}}}=\mathbb P^{\hat\pi}$ on $\{s_1^i\}_{i\in[n]}$, and since $\phi_D$ depends on $\mathbb P$ only through these initial states, $\mathbb P^{\hat\pi}$ also attains $\max_{\mathbb P\in\mathcal P}\inf_{R\in\mathcal R_\Pi}\phi_D(\mathbb P,R)$.
\end{proof}

Now we connect AsymRE to SFT+VM by computing the gradient. Writing $\bar R^*_{\pi,\beta}(s_1^i):=\frac{1}{m}\sum_{j=1}^m R^*_{\pi,\beta}(\tau^{ij})$ and using $\nabla R^*_{\pi,\beta}(\tau^{ij})=-\beta\sum_{h=1}^H\nabla\log\pi_h(a_h^{ij}|s_h^{ij})$, we have
\begin{equation*}
\begin{aligned}
   \nabla \mathcal E_{D,\textrm{VM}}(\pi)&=\frac{2}{n(m-1)}\sum_{i=1}^n\sum_{j=1}^m \left(R^*_{\pi,\beta}(\tau^{ij})-\bar R^*_{\pi,\beta}(s_1^{i})\right)\nabla R^*_{\pi,\beta}(\tau^{ij})\\
   &=-\frac{2\beta}{n(m-1)}\sum_{i=1}^n\sum_{j=1}^m\left(R^*_{\pi,\beta}(\tau^{ij})-\bar R^*_{\pi,\beta}(s_1^i)\right)\sum_{h=1}^H\nabla\log \pi_h(a_h^{ij}|s_h^{ij}),\\
    \nabla \mathcal L_{D,\textrm{SFT+VM}}(\pi)&=-\beta\cdot\frac{1}{nm}\sum_{i=1}^n\sum_{j=1}^m\sum_{h=1}^H \nabla\log\pi_h(a_h^{ij}|s_h^{ij})+\eta\nabla \mathcal E_{D,\textrm{VM}}(\pi)\\
    &=-\frac{2\eta\beta}{n}\sum_{i=1}^n\frac{1}{m-1}\sum_{j=1}^m\left(R^*_{\pi,\beta}(\tau^{ij})-\bar R^*_{\pi,\beta}(s_1^i)+\frac{m-1}{2m\eta}\right)\\
    &\qquad\times\sum_{h=1}^H\nabla\log \pi_h(a_{h}^{ij}|s_h^{ij}),
\end{aligned}
\end{equation*}
where the first equation is because $\sum_{j=1}^m(R^*_{\pi,\beta}(\tau^{ij})-\bar R^*_{\pi,\beta}(s_1^i))=0$, so the term with $\nabla\bar R^*_{\pi,\beta}(s_1^i)$ vanishes. Comparing the first two lines with the definition of $\nabla\mathcal E_{D,\textnormal{GRPG}}$ gives
\begin{equation}\label{eq:propto_gra_GRPG_VM}
    \nabla \mathcal E_{D,\textrm{VM}}(\pi)=2\beta\nabla\mathcal E_{D,\textrm{GRPG}}(\pi).
\end{equation}
Notice that if $\delta=\frac{m-1}{2m\eta}=\Theta(\frac{1}{\sqrt n})$, then it holds,
\begin{equation}\label{eq:propto_gra_AsymRE_SFT+VM}
    \nabla   \mathcal L_{D,\textrm{SFT+VM}}(\pi)=2\eta\beta\nabla\mathcal E_{D,\textrm{AsymRE}}(\pi).
\end{equation}

Now we can directly prove \cref{thm:bound_single_policy_coverage}.

\paragraph{Proof of \cref{thm:bound_single_policy_coverage}}

\begin{proof}
    Condition (ii) reads $\sup_{\pi\in\Pi}\|R_\pi\|_\infty\leq B$, so by \cref{prop:maxmin_closed_form}, $\mathbb P^{\hat\pi}\in\argmax_{\mathbb P\in\mathcal P}\inf_{R\in\mathcal R_\Pi}\phi_D(\mathbb P,R)$. We now apply \cref{thm:guarantee_from_max_min_game} with $\mathcal R=\mathcal R_\Pi$, whose assumption (1) is again condition (ii). By \cref{eq:trajectory_reward_KL_optimality}, $R_{\pi^*}(\tau)=R^*(\tau)-V^*_{\beta,1}(s_1)$, so condition (i) gives $\bar R=R_{\pi^*}\in\mathcal R_\Pi$. Finally, $\hat\pi\in\Pi=\Pi_{\mathcal R_\Pi}$ and $N_\infty(\mathcal R_\Pi,\beta/n)=N(\Pi,1/n)$ by \cref{lemma:policy_induced_reward}. The claim follows.
\end{proof}

\subsection{Discussion of \texorpdfstring{\cref{assum:geometric_mixture_closure}}{Assumption 2}}\label{sec:discussion_assumption2}
\cref{assum:geometric_mixture_closure} is what allows the min-max problem solved by AsymRE to be swapped into the max-min game in \cref{prop:maxmin_closed_form}; it is the policy-side counterpart of the convexity of the reward class in \citet{liu2024provably}. We discuss when it holds and what it costs.

\paragraph{Reward-induced policy classes}
If one starts from a convex class $\mathcal R$ of additive rewards, then $\Pi_{\mathcal R}$ satisfies the assumption automatically: by \cref{lemma:KL_optimal_trajectory_sol}, $\mathbb P^{\pi_R}(\tau|s_1)\propto\mathbb P^{\pi_{\textnormal{ref}}}(\tau|s_1)\exp(R(\tau)/\beta)$, so the sequence-level geometric mixture of $\mathbb P^{\pi_{R_1}}$ and $\mathbb P^{\pi_{R_2}}$ is $\mathbb P^{\pi_{\bar R}}$ with $\bar R=\alpha R_1+(1-\alpha)R_2\in\mathcal R$. However, $\Pi_{\mathcal R}$ is not a tractable policy class: evaluating a single $\pi_R$ requires the backward recursion in \cref{eq:reward_induced_policy}, i.e., computing $V_{R,h}(s)$ at every state reachable at step $h$. In the autoregressive setting there are $|\mathcal A|^{h-1}$ such states for each $s_1$, so the cost is of order $|\mathcal A|^H$, the same as summing over all trajectories to compute $Z_R(s_1)$.

\paragraph{Token-level parameterizations}
Starting instead from the policy side, for a general parametric class with explicit per-step policies $\pi_{\theta,h}(a|s)$ and $H>1$, the assumption may not hold. When $H=1$ it reduces to closure under per-step geometric mixtures, which holds for any exponential family with a convex parameter set, the setting of \citet{liu2024provably}. We do not know how to remove the assumption within the present framework, and closing this gap is left as an open problem.

\subsection{Proof of \texorpdfstring{\cref{prop:equi_between_zhaoran_and_maxmin}}{Proposition 4}}

\begin{lemma}\label{lemma:pointwise_traj_to_value}
    For fixed $\pi, f$, it holds,
    \begin{equation*}
        \mathcal L_\mu(\pi, R^\pi_f)=\tilde{\mathcal L}_\mu(\pi,f),\ {\mathcal E}_\mu(R^\pi_f)=\tilde{\mathcal E}_\mu(\pi, f).
    \end{equation*}
\end{lemma}

\begin{proof}
    The second equation is straightforward, just plug in. To prove the first equation, notice that,
    \begin{align*}
        &\mathbb E_{s_1\sim \rho}\left[\mathbb E_{\tau\sim \mathbb P^\pi(\cdot|s_1)}[R^\pi_f(\tau)]-\beta\textrm{KL}(\mathbb P^\pi(\cdot|s_1)\|\mathbb P^{\pi_{\textrm{ref}}}(\cdot|s_1))\right]\\
        &=\mathbb E_{s_1\sim \rho, \tau\sim \mathbb P^\pi(\cdot|s_1)}\left[\sum_{h=1}^H\left(f_h(s_h,a_h)-\beta\log\frac{\pi_h(a_h|s_h)}{\pi_{\textrm{ref},h}(a_h|s_h)}-V^\pi_{f,h+1}(s_{h+1})\right)\right]\\
        &=\mathbb E_{\pi}\left[\sum_{h=1}^H \left(V^\pi_{f,h}(s_h)-V^\pi_{f,h+1}(s_{h+1})\right)\right]=\mathbb E_{s_1\sim\rho}[V^\pi_{f,1}(s_1)],
    \end{align*}
    where in the second equation we use the definition of $V^{\pi}_{f,h}$. Similarly, we have,
    \begin{align*}
        &\mathbb E_{s_1\sim \rho}\left[\mathbb E_{\tau\sim \mathbb P^\mu(\cdot|s_1)}[R^\pi_f(\tau)]-\beta\textrm{KL}(\mathbb P^\mu(\cdot|s_1)\|\mathbb P^{\pi_{\textrm{ref}}}(\cdot|s_1))\right]\\
        &=\mathbb E_\mu\left[\sum_{h=1}^H\left(V^\mu_{f,h}(s_h)-V^\pi_{f,h+1}(s_{h+1})\right)\right]\\
        &=\mathbb E_{s_1\sim\rho}[V_{f,1}^\pi(s_1)]+\mathbb E_\mu\left[\sum_{h=1}^H\left(V^\mu_{f,h}(s_h)-V^\pi_{f,h}(s_h)\right)\right].
    \end{align*}
    Then,
    \begin{align*}
        \mathcal L_\mu(\pi, R^\pi_f)&=\mathbb E_{s_1\sim\rho}[V^\pi_{f,1}(s_1)]-\left(\mathbb E_{s_1\sim\rho}[V_{f,1}^\pi(s_1)]+\mathbb E_\mu\left[\sum_{h=1}^H\left(V^\mu_{f,h}(s_h)-V^\pi_{f,h}(s_h)\right)\right]\right)\\
        &=\mathbb E_\mu\left[\sum_{h=1}^H\left(V^\pi_{f,h}(s_h)-V^\mu_{f,h}(s_h)\right)\right]=\tilde{\mathcal L}_\mu(\pi,f).
    \end{align*}
\end{proof}

\paragraph{Proof of \texorpdfstring{\cref{prop:equi_between_zhaoran_and_maxmin}}{Proposition 4}}

\begin{proof}
 Fix any $\pi\in \Pi$. By assumption,
\begin{equation*}
    \mathcal R(\pi,\mathcal F)=\{R_f^\pi: f\in \mathcal F\}=\mathcal R.
\end{equation*}
Therefore,
\begin{equation*}
    \min_{R\in \mathcal R}\left(\mathcal L_\mu(\pi,R)+\eta {\mathcal E}_\mu(R)\right)
    =
    \min_{f\in \mathcal F}\left(\mathcal L_\mu(\pi,R_f^\pi)+\eta {\mathcal E}_\mu(R_f^\pi)\right).
\end{equation*}
By the construction of $R_f^\pi$ and \cref{lemma:pointwise_traj_to_value}, we have
\begin{equation*}
    \mathcal L_\mu(\pi,R_f^\pi)=\tilde{\mathcal L}_\mu(\pi,f),\ 
    {\mathcal E}_\mu(R_f^\pi)=\tilde{\mathcal E}_\mu(\pi,f).
\end{equation*}
Hence,
\begin{equation*}
    \min_{R\in \mathcal R}\left(\mathcal L_\mu(\pi,R)+\eta {\mathcal E}_\mu(R)\right)
    =
    \min_{f\in \mathcal F}\left(\tilde{\mathcal L}_\mu(\pi,f)+\eta \tilde{\mathcal E}_\mu(\pi,f)\right).
\end{equation*}
Since this equality holds for every $\pi\in \Pi$, the two outer maximization problems have the same objective function on $\Pi$. Therefore,
\begin{equation*}
    \argmax_{\pi\in \Pi}\min_{R\in\mathcal R} \left(\mathcal L_\mu(\pi, R)+\eta {\mathcal E}_\mu(R)\right)
    =
    \argmax_{\pi\in\Pi}\min_{f\in \mathcal F}\left(\tilde{\mathcal L}_\mu(\pi, f)+\eta\tilde{\mathcal E}_\mu(\pi,f)\right).
\end{equation*}
This completes the proof.
\end{proof}

\section{Experiments}
\label{app:synthetic}
We test the prediction that separates \cref{thm:bound_all_policy_coverage} from \cref{thm:bound_single_policy_coverage}: GRPG
($\delta=0$) needs all-policy coverage, whereas AsymRE ($\delta>0$) is
pessimistic and needs only single-policy coverage. All objectives, gradients,
and metrics are computed exactly.

\subsection{Setup}

\paragraph{Environment}
All experiments use the same MDP, a deterministic complete binary tree with
horizon $H=6$. The states at step $h$ are the $2^{h-1}$ nodes at depth $h$,
indexed so that $s\in\{2^{h-1},\dots,2^h-1\}$, and the initial state is the
root $s_1=1$. The action space is $\mathcal{A}=\{0,1\}$, and the transition
is deterministic: $T(s,a)=2s+a$ moves to the left or right child. Each
internal state $s$ has a preferred action $a^+(s)$ with reward
$r^*(s,a^+(s))=+1$, and the other action has reward $-1$. The tree has
$2^H-1=63$ internal states and $2^H=64$ distinct trajectories. The reference
policy is uniform, $\pi_{\mathrm{ref}}(\cdot\mid s)=(1/2,1/2)$. The behavior
policy takes the preferred action with probability $1-\epsilon$, i.e.\
$\mu(a^+(s)\mid s)=1-\epsilon$. Every subtree rooted at the same depth has the
same soft value. The KL-regularized optimal policy is therefore identical at
every state, with $\pi^*(a^+(s)\mid s)=(1+e^{-2/\beta})^{-1}$; for example,
this probability is $0.9933$ for $\beta=0.4$. Hence $\epsilon$ controls how
far $\mu$ is from $\pi^*$. We report two coverage coefficients, both computed
exactly by forward propagation of occupancies. The first is the all-policy
coefficient
$C_{\mathrm{all}}:=1+\sup_\pi\chi^2(d_H^\pi\|d_H^\mu)=\epsilon^{-H}$,
whose supremum is attained by the policy that always takes the non-preferred
action; it enters \cref{thm:bound_all_policy_coverage}. The second is the single-policy coefficient
$C_{\pi^*}:=\sum_{h=1}^H\chi^2(d_h^{\pi^*}\|d_h^\mu)$, where $d_h^\pi$ is the
state-action occupancy at step $h$; it enters \cref{thm:bound_single_policy_coverage} with comparator
$\pi^*$.

\paragraph{Experimental details}
\emph{Offline data.} Each dataset $\mathcal{D}$ consists of $n$ groups of
$m=8$ trajectories, sampled i.i.d.\ from $\mu$ starting at $s_1$. We use
$n=128$ unless stated otherwise.  We use a tabular softmax policy
$\pi_\theta(a\mid s)\propto\exp(\theta_{s,a})$, with the logits projected
onto $|\theta_{s,a}|\le 8$ for compactness. 
\emph{Methods.} GRPG minimizes the empirical variance objective
$\mathcal{E}_{\mathcal{D},\mathrm{VM}}$. AsymRE with shift $\delta$ minimizes
$\mathcal{E}_{\mathcal{D},\mathrm{VM}}+\frac{2m\beta\delta}{m-1}\,\mathcal{L}_{\mathcal{D},\mathrm{SFT}}$,
whose gradient is $2\beta\,\nabla\mathcal{E}_{\mathcal{D},\mathrm{AsymRE}}$ by
\cref{eq:propto_gra_AsymRE_SFT+VM}. \cref{thm:bound_single_policy_coverage} prescribes
$\delta_{\mathrm{thm}}=\frac{m-1}{2m\sqrt n}$. Behavior cloning (BC), the
limit $\delta\to\infty$, is the empirical
maximum-likelihood policy, uniform at unvisited states.
\emph{Optimization.} Policies are initialized at $\pi_{\mathrm{ref}}$. We
train with full-batch projected Adam  for at most $3000$ steps, or until
$\|\nabla\|_\infty<10^{-10}$. The objective depends on
$\mathcal{D}$ only through the counts of the $64$ trajectories and of the
within-group trajectory pairs, so the cost per step does not depend on $n$.
\emph{Evaluation.} We report the KL-regularized regret
$J_\beta(\pi^*)-J_\beta(\hat\pi)$, computed exactly by dynamic programming.
Tables report the median over independent datasets, with $5\%$--$95\%$
quantiles in brackets. Regret is in units of $10^{-3}$. All tables are
reproduced exactly, in about two minutes on a CPU.

\subsection{Ablations and Results}

\paragraph{Baseline shift (\cref{tab:synthetic-shift})}
We fix $\epsilon=0.03$, $\beta=0.4$, and $n=128$, and vary the shift
$\delta$. To measure pessimism directly, \cref{tab:synthetic-shift}
also reports the unsupported mass
$U(\hat\pi):=\sum_h\Pr_{\hat\pi}(N(s_h,a_h)=0)$, the expected number of steps
at which $\hat\pi$ takes an action never observed in $\mathcal{D}$.
The regret is U-shaped in $\delta$. Without a shift, GRPG puts substantial
probability on unobserved actions ($U=0.85$) and incurs large regret. Any
$\delta>0$ removes this failure; at $\delta=0.1\,\delta_{\mathrm{thm}}$ the regret is already smaller by a
factor of about $500$. The minimum lies around the prescribed value $\delta_{\mathrm{thm}}$.
Larger shifts drive $U$ toward zero but bias the solution toward $\mu$, and
the regret approaches that of BC.

\begin{table}[h]
  \centering
  \caption{Baseline-shift ablation ($\epsilon=0.03$, $\beta=0.4$, $n=128$;
  50 datasets), for which $\delta_{\mathrm{thm}}=0.0387$. Regret and $U$ are in
  units of $10^{-3}$.}
  \label{tab:synthetic-shift}
  \small
  \begin{tabular}{lrrrr}
    \toprule
    Method & $\delta$ & $\delta/\delta_{\mathrm{thm}}$ & Regret $\downarrow$ & $U(\hat\pi)$ \\
    \midrule
    GRPG               & $0$      & $0$      & $1366$ \;[1180, 1523]            & $853$ \\
    AsymRE             & $0.0039$ & $0.1$    & $2.69$ \;[2.19, 3.12]            & $3.33$ \\
    AsymRE             & $0.0193$ & $0.5$    & $0.314$ \;[0.216, 0.385]         & $0.576$ \\
    \textbf{AsymRE (Thm.~5)} & $0.0387$ & $1$ & $\mathbf{0.239}$ \;[0.178, 0.276] & $0.279$ \\
    AsymRE             & $0.0773$ & $2$      & $0.414$ \;[0.335, 0.483]         & $0.138$ \\
    AsymRE             & $0.309$  & $8$      & $4.06$ \;[3.26, 4.77]            & $0.032$ \\
    AsymRE             & $3.87$   & $100$    & $42.5$ \;[34.5, 51.8]            & $0.003$ \\
    BC                 & $\infty$ & $\infty$ & $57.4$ \;[46.9, 70.5]            & $0$ \\
    \bottomrule
  \end{tabular}
\end{table}

\paragraph{All-policy coverage (\cref{tab:synthetic-coverage})}
We fix $\beta=0.8$ (so $\pi^*$ is fixed) and $n=128$, and vary the behavior
parameter $\epsilon$. A smaller $\epsilon$ means that $\mu$ rarely takes the
non-preferred action, so the all-policy coefficient
$C_{\mathrm{all}}=\epsilon^{-H}$ grows rapidly as $\epsilon$ decreases, from
$1.4\times10^{3}$ at $\epsilon=0.3$ to $10^{12}$ at $\epsilon=0.01$. GRPG's
regret grows accordingly, from $0.012$ to $1.33$, as predicted by \cref{thm:bound_all_policy_coverage}.
AsymRE's regret stays at least an order of magnitude below GRPG's (except at
$\epsilon=0.3$, where both are small) and does not track $C_{\mathrm{all}}$.
Since changing $\epsilon$ also changes how well $\mu$ covers $\pi^*$,
AsymRE's regret varies with $C_{\pi^*}$ instead; it is largest at
$\epsilon=0.01$, where $C_{\pi^*}$ is largest.

\begin{table}[h]
  \centering
  \caption{All-policy coverage ablation ($\beta=0.8$, $n=128$; 30 datasets).
  Regret is in units of $10^{-3}$.}
  \label{tab:synthetic-coverage}
  \small
  \begin{tabular}{lrrrr}
    \toprule
    $\epsilon$ & $C_{\mathrm{all}}$ & $C_{\pi^*}$ & GRPG regret $\downarrow$ & AsymRE regret $\downarrow$ \\
    \midrule
    $0.01$ & $10^{12}$          & $19.8$  & $1329$ \;[1170, 1395] & $\mathbf{59.1}$ \;[51.6, 65.6] \\
    $0.1$  & $10^{6}$           & $0.14$  & $905$ \;[777, 1138]   & $\mathbf{2.84}$ \;[1.44, 4.56] \\
    $0.2$  & $1.6\times10^{4}$  & $2.38$  & $253$ \;[119, 343]    & $\mathbf{0.890}$ \;[0.657, 1.27] \\
    $0.3$  & $1.4\times10^{3}$  & $7.58$  & $12.2$ \;[0.00, 35.7] & $\mathbf{5.06}$ \;[3.70, 6.42] \\
    \bottomrule
  \end{tabular}
\end{table}

\paragraph{KL parameter and single-policy coverage (\cref{tab:synthetic-kl})}
We now fix $\epsilon=0.01$, so the behavior policy $\mu$, and hence
$C_{\mathrm{all}}=10^{12}$, is fixed. We vary $\beta$ instead. A larger
$\beta$ pulls $\pi^*$ toward the uniform $\pi_{\mathrm{ref}}$, so $\pi^*$
takes the non-preferred action more and more often: with probability
$6.7\times10^{-3}$ at $\beta=0.4$, but $0.22$ at $\beta=1.6$. Because $\mu$
takes this action only with probability $0.01$, the actions that $\pi^*$
uses become increasingly rare in the data. As a result, the single-policy
coefficient $C_{\pi^*}$ grows from $0.023$ at $\beta=0.4$ to
$3.6\times10^{4}$ at $\beta=1.6$. GRPG performs poorly throughout, since
$C_{\mathrm{all}}$ is huge for every $\beta$. AsymRE's regret grows with
$C_{\pi^*}$, as predicted by \cref{thm:bound_single_policy_coverage}. It is negligible while $\pi^*$ is well
covered ($\beta\le0.4$), and at $\beta=1.6$, where single-policy coverage
also fails, AsymRE is no better than GRPG.

\begin{table}[h]
  \centering
  \caption{KL-parameter ablation ($\epsilon=0.01$, so
  $C_{\mathrm{all}}=10^{12}$; $n=128$; 20 datasets). $\pi^*(\bar a)$ is the
  probability that $\pi^*$ takes the non-preferred action. Regret is in units
  of $10^{-3}$.}
  \label{tab:synthetic-kl}
  \small
  \begin{tabular}{lrrrr}
    \toprule
    $\beta$ & $\pi^*(\bar a)$ & $C_{\pi^*}$ & GRPG regret $\downarrow$ & AsymRE regret $\downarrow$ \\
    \midrule
    $0.2$ & $4.5\times10^{-5}$ & $0.21$            & $886$ \;[676, 991]    & $\mathbf{0.278}$ \;[0.205, 0.358] \\
    $0.4$ & $6.7\times10^{-3}$ & $0.023$           & $1538$ \;[1329, 1628] & $\mathbf{0.240}$ \;[0.217, 0.278] \\
    $0.8$ & $0.076$            & $19.8$            & $1321$ \;[1199, 1397] & $\mathbf{58.4}$ \;[51.6, 64.5] \\
    $1.6$ & $0.22$             & $3.6\times10^{4}$ & $782$ \;[714, 843]    & $769$ \;[693, 831] \\
    \bottomrule
  \end{tabular}
\end{table}

\newpage
\paragraph{Large $n$ (\cref{tab:synthetic-large-n})}
Although $\delta_{\mathrm{thm}}\to0$ as $n$ grows, the benefit of the shift
does not vanish over the range of $n$ we simulate. We fix $\epsilon=0.001$ and $\beta=0.4$ and increase $n$.
GRPG's regret stays around $1.5$ up
to $n=10^6$. The reason is that the empirical GRPG objective has many
minimizers: in every run we inspected, GRPG and $\pi^*$ both attain
$\mathcal{E}_{\mathcal{D},\mathrm{VM}}<10^{-13}$. The data constrain only the
observed trajectories, so probability can be moved toward unobserved actions
without changing the loss. Any $\delta>0$ breaks this tie in favor of the
policy supported by the data, regardless of the size of $\delta$.

\begin{table}[h]
  \centering
  \caption{Large-$n$ regime ($\epsilon=0.001$, $\beta=0.4$; 5 datasets per
  $n$). Regret is in units of $10^{-3}$.}
  \label{tab:synthetic-large-n}
  \small
  \begin{tabular}{lrrr}
    \toprule
    $n$ & $\delta_{\mathrm{thm}}$ & GRPG regret $\downarrow$ & AsymRE regret $\downarrow$ \\
    \midrule
    $10^{2}$ & $4.4\times10^{-2}$ & $1532$ \;[1470, 2778] & $\mathbf{3.27}$ \;[3.24, 7.91] \\
    $10^{4}$ & $4.4\times10^{-3}$ & $1497$ \;[1427, 1561] & $\mathbf{0.262}$ \;[0.233, 0.269] \\
    $10^{6}$ & $4.4\times10^{-4}$ & $1534$ \;[1533, 1542] & $\mathbf{0.186}$ \;[0.180, 0.197] \\
    \bottomrule
  \end{tabular}
\end{table}

\end{document}